\documentclass[]{bytedance_seed}

\usepackage[toc,page,header]{appendix}

\usepackage{wrapfig}

\usepackage{xcolor}
\usepackage{colortbl}
\usepackage[ruled,vlined,linesnumbered]{algorithm2e}
\usepackage{pifont}

\usepackage{minitoc}
\usepackage{amsmath,amssymb,amsfonts,mathrsfs}

\newtheorem{definition}{Definition}
\newtheorem{proof}{Proof}
\newtheorem{proposition}{Proposition}

\title{ Scaling Linear Attention with Sparse State Expansion  }

\affiliation[1]{ByteDance Seed}
\affiliation[2]{Institute of Automation, Chinese Academy of Sciences}
\affiliation[3]{The Hong Kong Polytechnic University}
\affiliation[4]{UC Santa Cruz}

\contribution{Full author list in \nameref{sec:Contributions}.}

\abstract{
The Transformer architecture, despite its widespread success, struggles with long-context scenarios due to quadratic computation and linear memory growth. While various linear attention variants mitigate these efficiency constraints by compressing context into fixed-size states, they often degrade performance in tasks such as in-context retrieval and reasoning. To address this limitation and achieve more effective context compression, we propose two key innovations. First, we introduce a \textbf{row-sparse update formulation} for linear attention by conceptualizing state updating as information classification. This enables sparse state updates via softmax-based top-$k$ hard classification, thereby extending receptive fields and reducing inter-class interference. Second, we present \textbf{Sparse State Expansion (SSE)} within the sparse framework, which expands the contextual state into multiple partitions, effectively decoupling parameter size from state capacity while maintaining the sparse classification paradigm. Supported by efficient parallelized implementations, our design achieves effective classification and highly discriminative state representations. We extensively validate SSE in both pure linear and hybrid (SSE-H) architectures across language modeling, in-context retrieval, and mathematical reasoning benchmarks. SSE demonstrates strong retrieval performance and scales favorably with state size. Moreover, after reinforcement learning (RL) training, our 2B SSE-H model achieves \textbf{state-of-the-art} mathematical reasoning performance among small reasoning models, scoring 64.5 on AIME24 and 50.2 on AIME25, significantly outperforming similarly sized open-source Transformers. These results highlight SSE as a promising and efficient architecture for long-context modeling.
}

\date{\today}
\correspondence{\email{lizheng.m@bytedance.com}, \email{guoqi.li@ia.ac.cn}}

\begin{document}
\maketitle


\begin{figure}[h]
    \centering
    \vspace{-3mm}
    \begin{subfigure}[t]{0.73\textwidth}
        \centering
        \includegraphics[width=\textwidth]{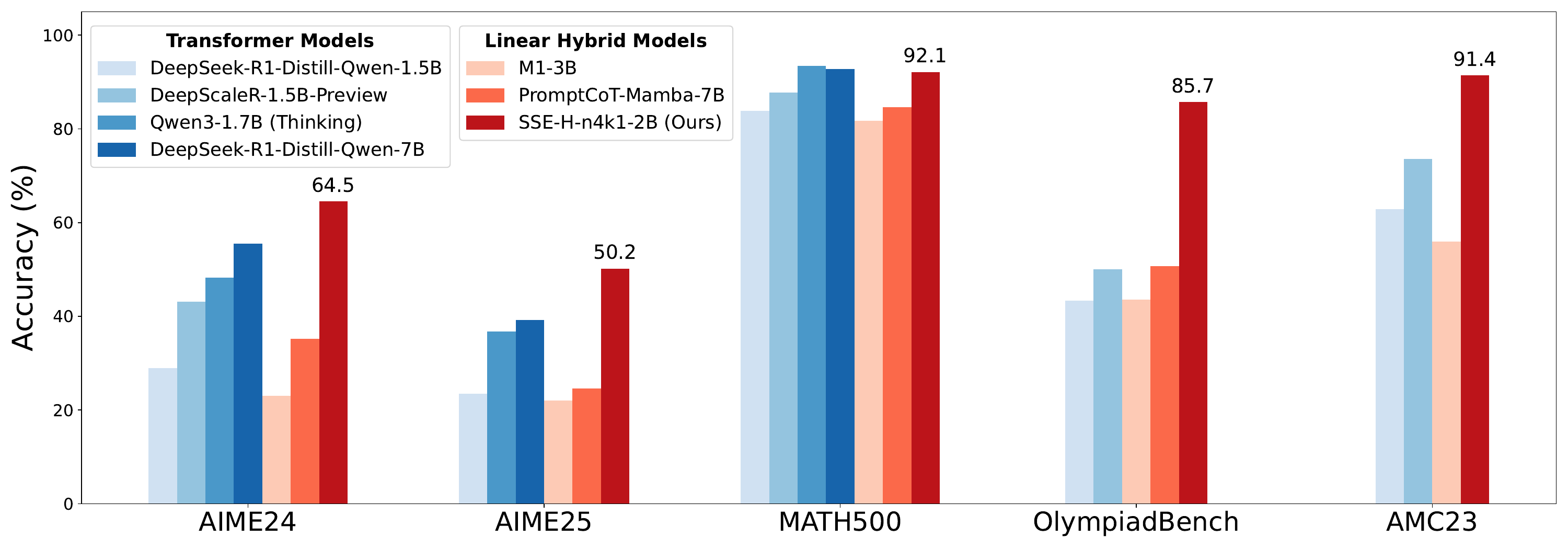}
    \end{subfigure}
    \hfill
    \begin{subfigure}[t]{0.25\textwidth}
        \centering
        \includegraphics[width=\textwidth]{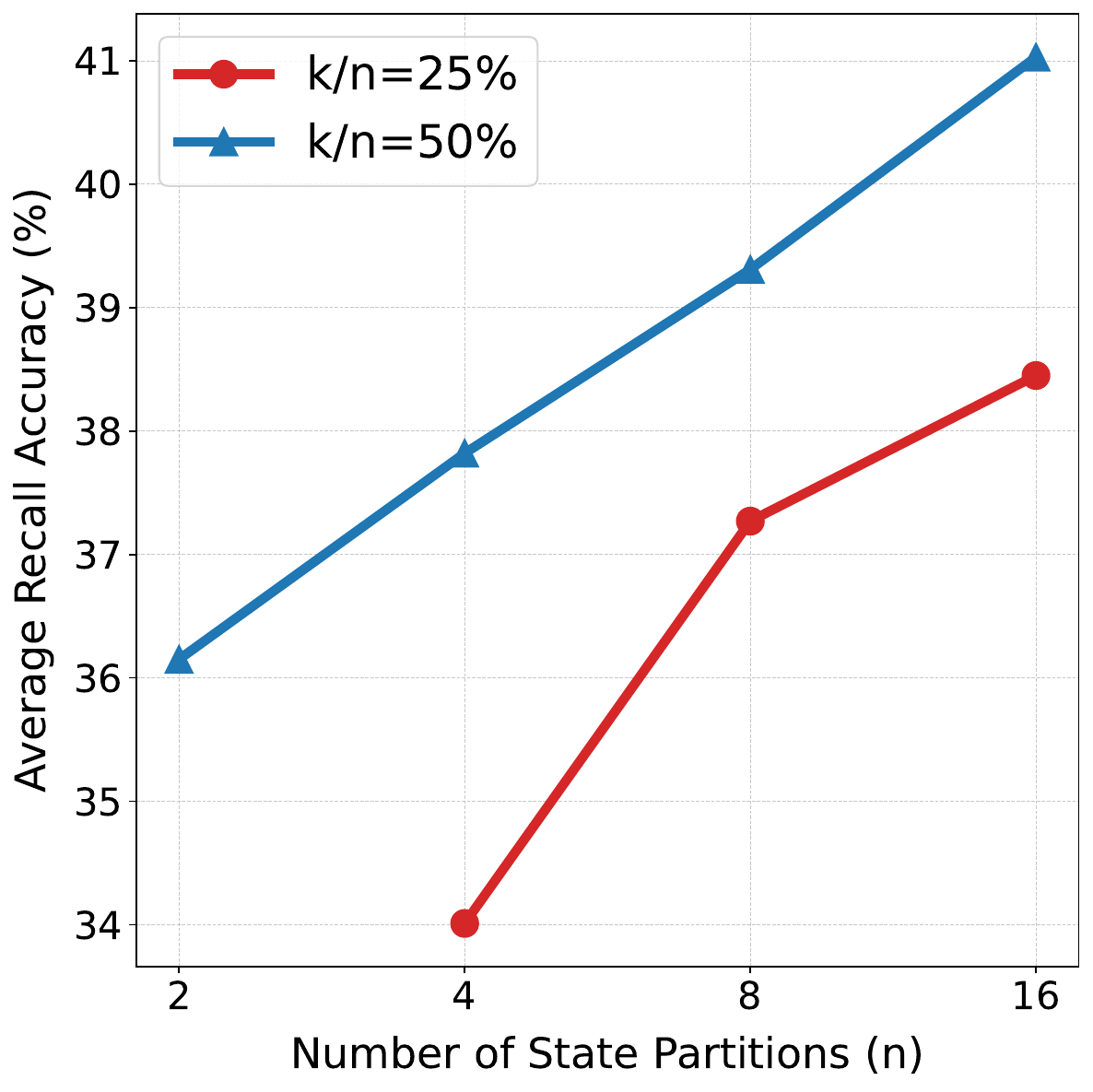}
    \end{subfigure}
    \caption{
        \textbf{Left:} Benchmark performance comparison in mathematical reasoning.
        \textbf{Right:} State scaling performance. $n$ represents the number of expanded state partitions, and $k$ denotes the top-$k$ selection size.
        }
    \label{fig:header}
\end{figure}

\section{Introduction}

The Transformer architecture~\citep{vaswani2017attention} has achieved remarkable success in various sequence modeling tasks, leveraging its attention mechanism for expressive modeling and high parallelism. However, it inherently struggles with processing long sequences due to quadratic computational complexity and linearly growing memory usage of key-value (KV) caches during inference.

To address these limitations, numerous linear attention variants have been proposed~\citep{katharopoulos2020transformers,yang2024gated,dao2024transformers,yang2024gdn}. Such approaches typically achieve sub-quadratic complexity and enable partial parallelism in training, along with RNN-like constant memory usage during decoding. However, most methods compress contextual information into fixed-size state matrices (e.g., $128\times d$), which often compromises performance on tasks such as in-context retrieval and reasoning\footnote{These tasks rely heavily on fine-grained token-level information flow, particularly the retrieval and extraction of relevant context for reasoning. As the proportion of tokens requiring fine-grained interactions increases, contextual compression becomes more challenging. }~\citep{arora2024zoology,arora2024simple,jelassi2024repeat,akyurek2024context}. This raises a key question: how can context compression be designed to balance modeling fidelity and computational efficiency in long-context scenarios?

Motivated by the pursuit of more effective context compression, we propose two key innovations. First, we introduce a row-sparse update formulation for linear attention by conceptualizing contextual state updating as information classification. Specifically, we interpret the key vector mapping $\mathbf{k}_t = f(\mathbf{x}_t, \mathbf{W}_k)$ as a classification function and employ a softmax-based top-$k$ hard classification strategy. At each step, this formulation selectively updates only state rows corresponding to predicted classes. Both theoretical analysis and synthetic experiments validate that our approach achieves larger receptive fields and reduces inter-class interference. Second, to mitigate the memory capacity constraints caused by limited state sizes, we extend the contextual state within the row-sparse update framework and propose Sparse State Expansion (SSE). SSE expands the state into $N$ partitions with shared attention parameters and utilizes a write-read gate for partition selection, followed by a softmax-based selection of state rows within partitions. This design effectively decouples parameter size from state capacity while maintaining the sparse classification paradigm, and empirical results indicate that the states of SSE achieves lower inter-row similarity and higher singular value entropy, which enables more discriminative state representations. We also propose efficient parallelized implementations of SSE, employing masking and the varlen technique optimized for various training contexts.

We extensively validate SSE under both linear and hybrid architectures (denoted as SSE-H), across three core capabilities: language modeling, in-context retrieval, and mathematical reasoning. SSE demonstrates strong language modeling performance among advanced linear models~\citep{yang2024gated,yang2024gdn,du2025mom}, while SSE-H outperforms both hybrid and Transformer baselines. On long-context retrieval tasks, SSE consistently improves upon other linear attention models, and its hybrid variant significantly narrows the gap with softmax attention Transformers. For reasoning, 2B SSE-H model achieves scores of 64.5 on AIME24 and 50.2 on AIME25 with advanced RL training~\citep{shao2024deepseekmath,yu2025dapo}—exceeding the best reported results from similarly sized open-source Transformers (Figure~\ref{fig:header}. Left). Moreover, SSE exhibits strong scalability with respect to state capacity (Figure~\ref{fig:header}. Right), and supports efficient conversion from pretrained Transformers at different scales, demonstrating its flexibility across model architectures and training regimes.

Our main contributions are as follows:
\begin{itemize}
   \item We introduce a \textbf{row-sparse state update framework} that conceptualizes state updating as information classification. This framework enables sparse state updates through softmax-based top-$k$ hard classification, supported by both theoretical and empirical analysis.
    \item Within the proposed framework, we introduce \textbf{Sparse State Expansion (SSE)}, an efficient state expansion mechanism designed to effectively manage parameter count and preserve the classification paradigm. Furthermore, we develop parallelized implementations of SSE, suitable for various training contexts.
    \item We extensively validate SSE and the hybrid SSE-H across diverse training stages and benchmarks. Our 2B hybrid model achieves \textbf{state-of-the-art reasoning performance} among small reasoning models, scoring 64.5 on AIME24 and 50.2 on AIME25—exceeding previous open-source softmax attention Transformers of comparable size.
\end{itemize}
\section{Preliminary}
Autoregressive language models aim to estimate the probability of a language sequence $s$, with the $t$-th token denoted as $\mathbf{x}_t$. This is achieved by modeling the conditional probability $P_{\theta}(\mathbf{x}_t|\mathbf{x}_{<t})$ and optimizing the cross-entropy loss through next-token prediction.
During inference, models generate $\mathbf{x}_t$ by sampling from the learned distribution $P_{\theta}(\mathbf{x}_t|\mathbf{x}_{<t})$, forming the basis of in-context learning (ICL)~\citep{olsson2022context}. Various token-mixers encode historical tokens $\mathbf{x}_{<t}$ into a compact contextual state $\mathbf{S}_t$, facilitating sampling via $P_{\theta}(\mathbf{x}_t|\mathbf{S}_{t-1}) \triangleq P_{\theta}(\mathbf{x}_t|\mathbf{x}_{<t})$\footnote{This formulation often enables more efficient inference. For example, Transformers utilize a KV cache to reduce redundant computations, while RNNs employ a hidden state for recurrent and constant inference cost.}. Consequently, a key distinction among language models lies in how they maintain and utilize the contextual state $\mathbf{S}_t \leftarrow{} \mathbf{x}_{<t}$, which is the central focus of this study.

\textbf{Softmax Attention.} Softmax attention~\citep{vaswani2017attention} stores the key-value vectors derived from historical tokens in a KV cache ($\mathbf{K}_t\in\mathcal{R}^{t\times d}$,$\mathbf{V}_t\in\mathcal{R}^{t\times d}$). As a new token $\mathbf{x}_t$ arrives, the corresponding vectors $\mathbf{k}_t,\mathbf{v}_t$ are appended to this cache. Attention is then computed by evaluating the interaction between the current token’s query $\mathbf{q}_t$ and the KV cache:
\begin{align}
        \mathbf{S}_t &= \{\mathbf{K}_t,\mathbf{V}_t\}, \\
        \mathbf{K}_t &= \begin{bmatrix}
        \mathbf{K}_{t-1}\\ \mathbf{k}_t
    \end{bmatrix},\  \mathbf{V}_t = \begin{bmatrix}
        \mathbf{V}_{t-1}\\ \mathbf{v}_{t}
    \end{bmatrix},\\
    \mathbf{o}_t &= \operatorname{softmax}(\mathbf{q}_{t}\mathbf{K}_t^{\top})\mathbf{V}_t.
\end{align}
This append-based mechanism for state representation and attention naturally enables high parallelism. However, as historical information is not compressed or pruned, it incurs efficiency bottlenecks when processing long sequences. Specifically, both memory and computational costs grow linearly with sequence length at each time step. When all queries are stacked and computed in parallel during training, the computational cost becomes quadratic.

\textbf{Vanilla Linear Attention.} Vanilla linear attention~\citep{katharopoulos2020transformers} compresses historical context into a fixed-size state matrix ($\mathbf{S}_t\in\mathcal{R}^{c\times d}$) using an outer-product update, where $c$ is a predefined constant (e.g., 128) independent of the sequence length:
\begin{align}
        \mathbf{S}_t &=  \mathbf{S}_{t-1}+\mathbf{k}_t^{\top}\mathbf{v}_t,\\
    \mathbf{o}_t &= \mathbf{q}_{t}\mathbf{S}_t.
\end{align}
For long contexts, this aggressive compression yields significantly higher inference efficiency, requiring only constant memory and computational cost per time step. However, when $c \ll t$, the lossy compression causes performance gaps compared to softmax attention, particularly in tasks such as in-context retrieval and reasoning. By reformulating the linear recurrence into a chunk-wise parallel form, linear attention enables hardware-efficient and sub-quadratic training~\citep{sun2023retentive,qin2023scaling,yang2024gated,dao2024transformers}.

Given the growing importance of training and inference efficiency and the inherent redundancy of natural language, effective contextual compression has become a central concern, particularly for long-sequence processing and generation. However, existing linear attention mechanisms have yet to fully exploit this potential. This leads to the central question of this work:
\begin{center}
\textit{How can we obtain a more effective compressed state than vanilla linear attention?}
\end{center}

To address this, our method introduces two key components. First, we derive a novel row-sparse update formulation for linear attention by conceptualizing state updating as information classification. Second, we propose Sparse State Expansion (SSE) for linear attention, designed to efficiently augment state capacity within the row-sparse update framework. The following section will elaborate on these contributions and detail their implementation.

\section{Row-Sparse State Update: An Information Classification Perspective}
Understanding the operational mechanism of contextual states is crucial for developing effective compression methods. To this end, we first present a recurrent outer-product update formulation that provides a unified view of attention modeling~\citep{chou2024metala}:
\begin{align}
    \mathbf{S}_t = \mathbf{\Lambda}_t \mathbf{S}_{t-1}+\phi(\mathbf{k}_t)^{\top}\mathbf{v}_t.\label{eq:unified_recurrent}
\end{align}
Here, $\mathbf{\Lambda}_t$ denotes the state transition matrix and $\phi$ is the feature map. This unified recurrent framework for state updates highlights two orthogonal directions for development:

(1) \textbf{The role of $\mathbf{\Lambda}_t$ in historical information management.} For softmax attention\footnote{For simplicity, we omit the normalization term here.}, $\mathbf{\Lambda}_t = \mathbf{I}$, which implies no inherent temporal prior and necessitates explicit positional embeddings. Conversely, many modern linear attention variants~\citep{sun2023retentive,qin2022devil,peng2023rwkv} incorporate a strong recency bias into $\mathbf{\Lambda}_t$, often implemented via exponentially decaying matrices with values in $(0,1)$. More recent works introduce input-dependent gating matrices~\citep{yang2024gated,gumamba,dao2024transformers} or diagonal-plus-low-rank (DPLR) matrices derived from the delta-rule~\citep{yangparallelizing,yang2024gdn,siems2025deltaproduct} to further enhance the management and transition of historical information\footnote{Gating mechanisms often employ diagonal (or scalar) matrices, $\mathbf{\Lambda}_t=\operatorname{diag}(\bm{\alpha}_t)$, where $\bm{\alpha}_t$ is input-dependent with values in $(0,1)$. DPLR matrices are typically structured as $\mathbf{\Lambda}_t=\operatorname{diag}(\bm{\alpha}_t)+\mathbf{a}_t^{\top}\mathbf{b}_t$}.

(2) \textbf{The role of $\phi(\mathbf{k}_t)$ in new information processing.} After projecting the input $\mathbf{x}_t$ via linear transformations, linear attention typically applies a feature map $\phi$ to embed it into a finite-dimensional space. Early designs of $\phi$ sought to approximate the exponential kernel of softmax attention~\citep{katharopoulos2020transformers,choromanskirethinking,peng2021random,zhanghedgehog}. Conversely, more recent variants commonly adopt identity mapping or simple activation functions like ReLU or SiLU~\citep{sun2023retentive,yang2024gated,chou2024metala,yangparallelizing,yang2024gdn}.

While the first direction, concerning structured matrices for more effective state transition, exhibits a clear trend, the second direction—identifying a more effective key vector mapping for enhanced state updating—remains unclear. In this section, we focus on the second direction, aiming to design a more effective $\phi(\mathbf{k}_t)$. By conceptualizing contextual state updating as information classification, we propose a novel row-sparse update formulation for linear attention. This is implemented via a top-$k$-then-softmax strategy, performing hard classification. Notably, our analysis is agnostic to the specific design of $\mathbf{\Lambda}_t$ and seamlessly integrates with various state transition mechanisms.

\begin{figure}
    \centering
    \includegraphics[width=\textwidth]{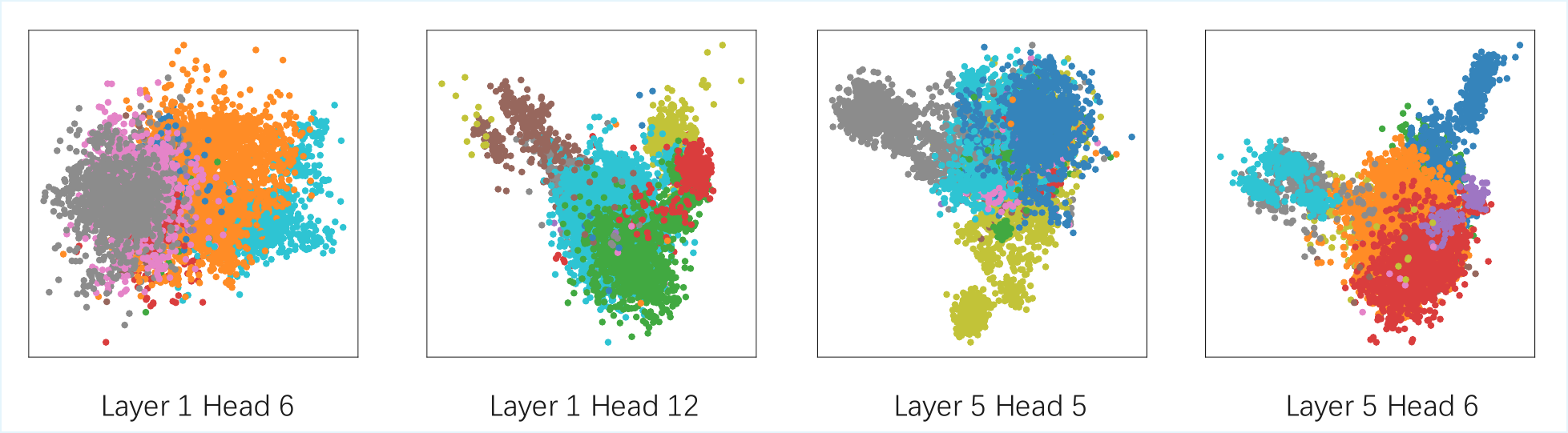}
    \caption{\textbf{Clustering of information within linear attention state rows.} We observe that learned state representations reveal clear clustering patterns. Specifically, we assign each token's value vector (represented as a point) to a specific state row (indicated by color) by taking the maximum activation over its corresponding key vector. This assignment demonstrates that information within the same row tends to share similar feature representations.}
    \label{fig:info_cluster}
\end{figure}

\subsection{Conceptualizing the State Updating as Information Classification}
In Equation~\ref{eq:unified_recurrent}, we conceptualize $\mathbf{k}_t = f(\mathbf{x}_t, \mathbf{W}_k)$ as a classification function\footnote{For clarity, $\mathbf{k}_t$ denotes the vector after the $\phi$ mapping throughout the remainder of this paper.}. Specifically, $f(\mathbf{x}_t, \mathbf{W}_k)$ classifies the input $\mathbf{x}_t$, and the classification results determine how information is distributed across different state slots (i.e., rows). A higher value of $\mathbf{k}_t$ for a given class corresponds to more information being assigned to its respective row. Consequently, each row of the state represents a distinct feature subspace, storing information associated with similar classification decisions. To enhance expressiveness, the stored information is typically computed as $\mathbf{v}_t = \mathbf{x}_t \mathbf{W}_v$ in practice. For instance, using a linear attention module with a linear classifier ($\mathbf{k}_t = \mathbf{x}_t\mathbf{W}_k$), we theoretically demonstrate that information assigned to the same state row exhibits similar features (Appendix~\ref{app:theory}, Proposition~\ref{prop:class}). To validate this, we analyze the information assignments of linear attention using a top-1 row-selection strategy. As Figure~\ref{fig:info_cluster} illustrates, the information composition of state rows exhibits a clustering pattern, further supporting our proposed classification framework. However, vanilla linear attention often suffers from inter-class mixing due to its lack of this explicit  hard assignment.

This perspective allows us to reinterpret various designs of $f(\mathbf{x}_t, \mathbf{W}_k)$. Initially, in modern linear attention models~\citep{sun2023retentive,yang2024gated,chou2024metala}, $\mathbf{k}_t$ was typically derived through a single linear transformation, specifically $f(\mathbf{x}_t, \mathbf{W}_k) = \mathbf{x}_t\mathbf{W}_k$, thereby serving as a linear classifier. More recent models~\citep{yangparallelizing,yang2024gdn} incorporate simple nonlinear activations, yielding more expressive nonlinear classifiers like $f(\mathbf{x}_t, \mathbf{W}_k)=\operatorname{SiLU}(\mathbf{x}_t\mathbf{W}_k)$. Moreover, more sophisticated classification architectures, such as gating mechanisms~\citep{gumamba,dao2024transformers,beckxlstm} or MLPs~\citep{zhanghedgehog,kasai2021finetuning}, are also employed. Beyond these examples, viewing state construction as an information classification process enables novel architectural designs, such as the softmax-based hard classification we detail in the next section.

\begin{figure}
    \centering
    \includegraphics[width=\textwidth]{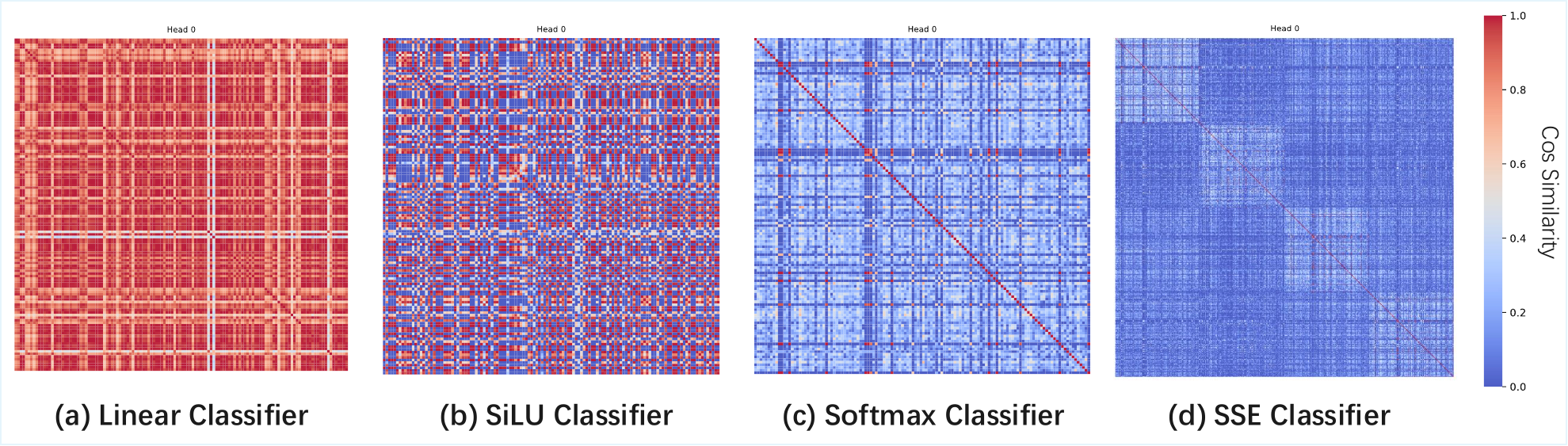}
    \caption{\textbf{Row-wise cosine similarity of contextual states in linear attention models with varying classifier designs.} The figure presents the $128 \times 128$ similarity matrices between state rows, where darker blue indicates lower similarity, reflecting more effective classification.}
    \label{fig:state_sim}
\end{figure}

To evaluate the effectiveness of different classification strategies, we analyze the row-wise cosine similarity of the contextual states across various models. As Figure~\ref{fig:state_sim}(a)(b) illustrates, the results highlight the advantages of more expressive classification functions while simultaneously revealing several limitations of existing linear attention mechanisms. Using a linear classifier as an illustrative example, we theoretically demonstrate the potential drawbacks of storing information from different classes in the same state rows within current linear attention designs (Appendix~\ref{app:theory}, Propositions~\ref{prop:noise1}-\ref{prop:forget}):
(1) Writing information to all rows at each time step introduces noise and interference from unrelated classes into the historical memory, thereby blurring the stored representations.
(2) In gated variants of linear attention, applying uniform decay to all state rows at each time step, driven by a strong recency bias, progressively leads to the forgetting of previously stored information. These issues stem from implicitly incorporating the notion of information classification without fully leveraging the resulting class assignments for compressed storage. These observations motivate the question: \textit{How can the contextual state in linear attention be more effectively updated and maintained from an information classification perspective?}

\subsection{Row-Sparse Update via Softmax-Based Hard Classification}
We posit that the aforementioned issues primarily stem from suboptimal designs of the classification function $f(\mathbf{x}_t, \mathbf{W}_k)$. From a classification standpoint, a natural and widely adopted choice is the softmax classification head, which takes the form $f(\mathbf{x}_t, \mathbf{W}_k)=\operatorname{softmax}(\mathbf{x}_t\mathbf{W}_k)$. 
Furthermore, our theoretical analysis (Appendix~\ref{app:theory}) indicates that updating only a subset of state rows based on class assignments improves the utilization of limited state capacity. This objective can be naturally achieved through top-$k$ hard classification. Our synthetic MQAR~\citep{arora2024zoology} experiment further demonstrates the benefits of top-$k$ updates for information recall (see Figure~\ref{fig:mqa_topk}). Combining these two insights, we propose a \textbf{row-sparse update formulation for linear attention}:
\begin{align}
    \mathbf{k}_t &= \operatorname{softmax}(\operatorname{top\text{-}}k(\mathbf{x}_t\mathbf{W}_k)),\label{eq:sec3_softmax}\\
    \mathbf{S}_t &= \mathbf{\Lambda}_t \mathbf{S}_{t-1}+\mathbf{k}_t^{\top}\mathbf{v}_t.
\end{align}
Softmax-based top-$k$ hard classification ensures that each contextual state row stores similar information, resulting in higher inter-row discriminability and more precise information organization. Additionally, each element undergoes fewer decay operations, effectively extending the receptive field over longer contexts. This, in turn, improves the model’s retrieval performance. As Figure~\ref{fig:state_sim}(c) illustrates, the low row-wise similarity observed in the states of softmax-based hard classifiers validates the effectiveness of our proposed framework.

To further enhance the stability of this approach and mitigate class imbalance, we promote a more uniform distribution of samples across state rows. Such balance stabilizes training and facilitates better parameter fitting throughout the state matrix. This can be enforced either via an auxiliary loss~\citep{fedus2022switch,lepikhingshard} promoting uniform row usage or through auxiliary-loss-free strategies~\citep{wang2024auxiliary}. In our implementation, we employ the auxiliary loss $\mathcal{L}_{balance}=\alpha\frac{c}{k}\sum_{i=1}^cf^i\cdot\mathbf{k}_{t}^i$, where $\alpha$ is a manually specified coefficient, $c$ is the number of state rows, and $f^i$ denotes the selection frequency of the $i$-th row\footnote{For SSE with expanded partitions, we use a partition-based loss instead of a row-based one.}. This objective, grounded in the class balance assumption, complements the sparsity induced by top-$k$ hard classification and enables joint optimization during training.

\section{SSE: State Expansion under the Row-Sparse Framework}
The most critical gap between linear attention and softmax attention is their performance on long-context retrieval and reasoning tasks, which has become a central focus in recent research. This limitation primarily arises from the restricted memory capacity caused by their small state size. For instance, under typical configurations (e.g., $c=128$), the memory budget of linear attention is roughly equivalent to that of softmax attention with a context window of only 64 tokens. Therefore, increasing the contextual state size in linear models is a key direction for addressing their current shortcomings. Recent works have explored this issue from various perspectives~\citep{arora2024simple,du2025mom,guo2025log,zhang2025test}. Building on this foundation, we pose the following question: \textit{How can state expansion be designed from the perspectives of information classification and row sparsity?}

Before addressing this question, we first consider a preliminary issue: the relationship between parameter size and state size. These represent two distinct yet potentially coupled dimensions of model scaling. In this work, we make the simplifying assumption of fixing the parameter size while independently examining the effect of state size expansion. This assumption is motivated by the observation that softmax attention employs $4d^2$ attention parameters to manage a state of size $2 \times t \times d$, yet still exhibits strong in-context learning capabilities when $t \gg c$. This suggests that, for the sequence lengths considered in this study, the existing parameter capacity is sufficient to manage compressed states effectively, obviating the need for increased parameters to fit smaller states. Therefore, as we investigate the expansion of linear attention from a baseline state of size $c \times d$, our goal is to isolate and analyze the benefits of state expansion without the confounding effects of increased parameter count.

\subsection{Sparse State Expansion via Information Classification}
\begin{figure}
    \centering
    \includegraphics[width=0.75\textwidth]{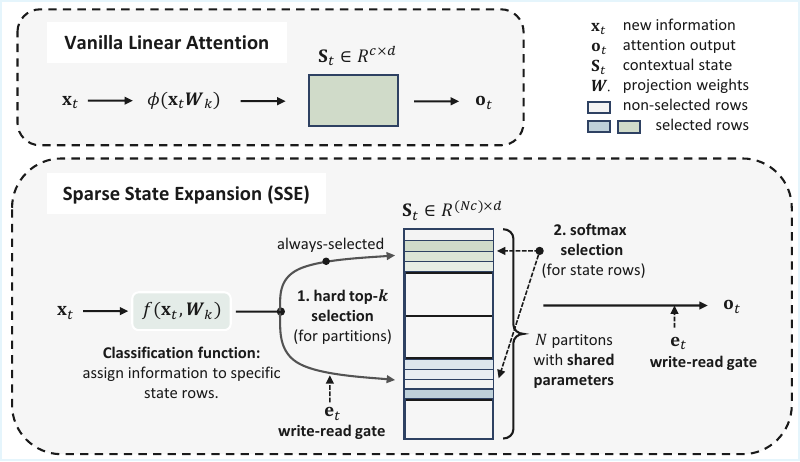}
\caption{\textbf{Comparison between vanilla linear attention and SSE.} SSE expands the state into $N$ partitions within the row-sparse update framework, where a classification function assigns information to specific state rows. All partitions share attention parameters. Sparse row selection follows two steps: (1) top-$k$ partition selection based on a write-read gate (blue indicates selected partitions; green marks an always-selected partition for training stability), and (2) row selection within the chosen partitions via softmax over key vectors.}
\label{fig:sse_arch}
\end{figure}
Within the information classification framework, a straightforward strategy involves expanding the number of classes by a factor of $N$, resulting in $N \times c$ total state rows organized into $N$ partitions. For each input token $\mathbf{x}_t$, we utilize the top-$k$-then-softmax mechanism to select a subset of rows from the chosen partitions. Building on our prior analysis, this design is driven by two key objectives:
(1) scaling the state size independently of parameter count, and
(2) preserving the sparse classification paradigm by explicitly performing information classification. For efficiency, we avoid applying top-$k$ directly across all state rows (Equation~\ref{eq:sec3_softmax}). Instead, we first perform hard top-$k$ selection over partitions, followed by soft row selection within the chosen partitions using softmax. Accordingly, our core design consists of the following components:

\textbf{Shared attention parameters across partitions.} We use a single set of projection weights (QKV) across all partitions. This design choice is equivalent to performing segmented clustering of input information along the sequence dimension, thereby enabling sparse token-level interactions.

\textbf{Write-read gate for hard partition selection, followed by soft row selection via softmax.}
A gate vector $\mathbf{e}_t \in \mathcal{R}^N$ is first used to select the top-$k$ partitions, and within these partitions $\operatorname{softmax}(\mathbf{k}_t)$ further selects the state rows. Only the selected partitions are updated, and ablation results show that the softmax step yields over 40\% row sparsity. Furthermore, the gate is applied to both state input (KV) and output (Q), thereby governing both information writing and reading.

We refer to the linear attention variant that integrates the two core components described above as \textbf{Sparse State Expansion (SSE)}, illustrated in Figure~\ref{fig:sse_arch}. The complete computational procedure is as follows:
\begin{align}
\mathbf{e}_t &= \operatorname{softmax}(\mathbf{x}_t\mathbf{W}_e),\label{eq:eta_proj}\\
\mathcal{T} &= \{i \ | \ \mathbf{e}_t^i \in \operatorname{top\text{-}}k(\mathbf{e}_t)\},\label{eq:sparse_index}\\
\mathbf{q}_t&=\mathbf{x}_t\mathbf{W}_q, \ \mathbf{v}_t =\mathbf{x}_t\mathbf{W}_v,\\
\mathbf{k}_t&=\operatorname{softmax}(\mathbf{x}_t\mathbf{W}_k),\label{eq:k_proj}\\
\mathbf{S}_{t}^i&=\begin{cases}
    \mathbf{\Lambda}_t \mathbf{S}_{t-1}^{i}+\mathbf{e}_t^i\cdot\mathbf{k}_t^{\top}\mathbf{v}_t, & \text{for}\ i\in \mathcal{T}\\
    \mathbf{S}_{t-1}^{i}, & \text{for}\ i\notin \mathcal{T}\\
\end{cases}\label{eq:state_update}\\
\mathbf{o}_t &= \sum_{i\in\mathcal{T}}\mathbf{e}_t^i\cdot\mathbf{q}_t\mathbf{S}_t^i.\label{eq:state_output}
\end{align}

(1) We compute a gate vector $\mathbf{e}_t \in \mathcal{R}^N$ through a linear transformation of the input followed by softmax (Equation~\ref{eq:eta_proj}), and select the top-$k$ partitions based on $\mathbf{e}_t$ (Equation~\ref{eq:sparse_index}). (2) To preserve the unified classification behavior, we apply softmax over all state rows of the selected partitions. Since all partitions share the same projection matrix $\mathbf{W}_k$, applying softmax over the $k \times c$ selected rows is equivalent to applying it independently within each partition (Equation~\ref{eq:k_proj}); the scaling factor can be absorbed into the parameters and thus omitted. (3) To ensure the gate remains trainable, we incorporate $\mathbf{e}_t$ into both state input (KV) and output (Q) (Equations~\ref{eq:state_update}–\ref{eq:state_output}), thereby controlling both writing and reading and enabling more targeted optimization. (4) For a selected partition $i$, the gated key vector is $\mathbf{k}_t^i = \mathbf{e}_t^i\cdot \operatorname{softmax}(\mathbf{x}_t\mathbf{W}_k)$ (Equation~\ref{eq:state_update}). By the normalization property of softmax, the gate can be embedded into the classification function while preserving consistent behavior, with classification intensity $\sum_i \sum_d (\mathbf{k}_t^i)_d = 1$. Only the selected partitions $\mathbf{S}_t^i$ are updated.  (5) Beyond state updating, we also perform selective reading based on information similarity. Within the write–read gating framework, we restrict reading to the top-$k'$ partitions using the same gate $\mathbf{e}_t$ (Equation~\ref{eq:state_output}): 
$\mathbf{o}_t = \sum_{i \in \mathcal{T'}} \mathbf{e}_t^i \cdot \mathbf{q}_t \mathbf{S}_t^i$, where by default $\mathcal{T'} = \mathcal{T}$.

\begin{wrapfigure}{r}{0.4\textwidth}
    \centering
    \includegraphics[width=0.35\textwidth]{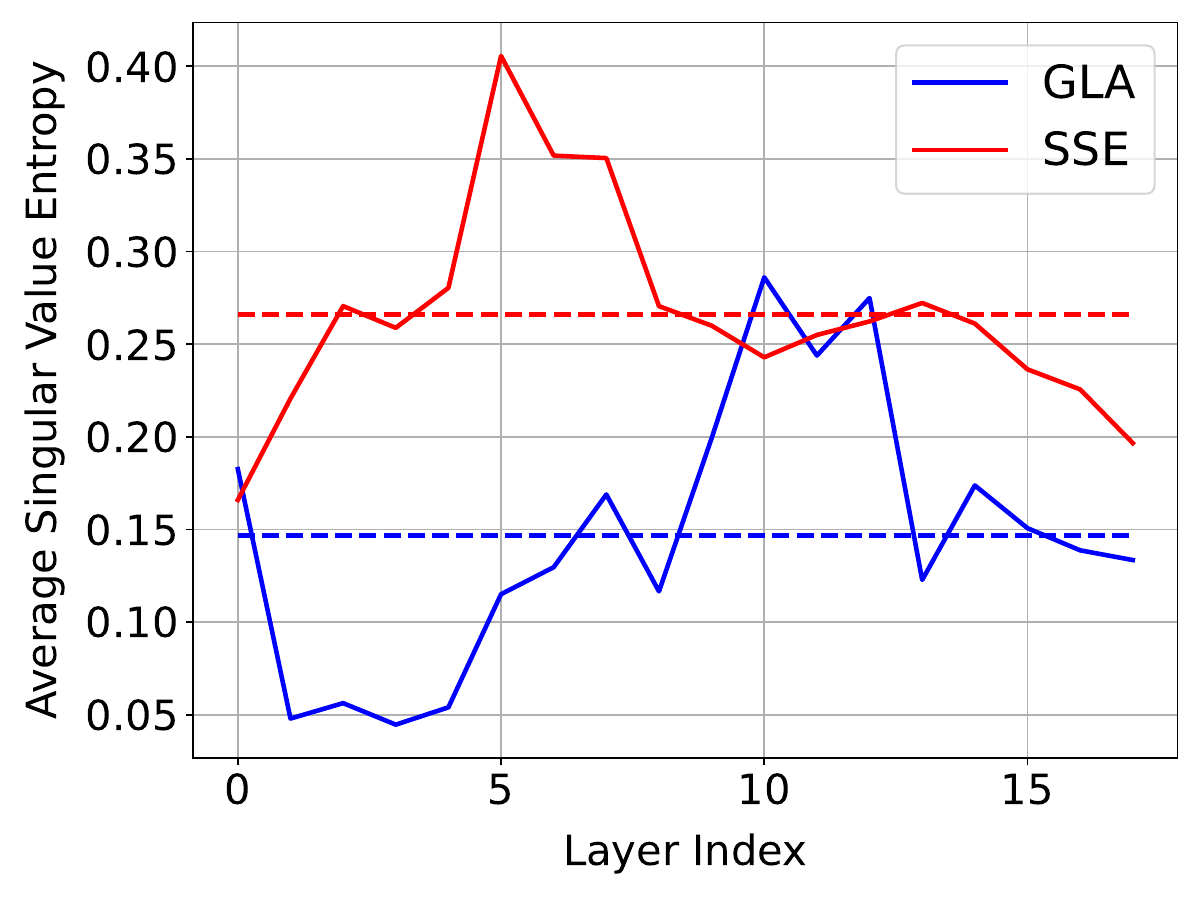}
    \caption{\textbf{Singular value entropy of the contextual states.} SSE with diagonal gating exhibits higher singular value entropy than GLA, indicating a less compressible and more effectively utilized state composition.}
    \label{fig:state_entropy}
\end{wrapfigure}
Local interactions between consecutive tokens serve as an effective prior for language modeling~\citep{fu2022hungry,peng2023rwkv,arora2024simple}, which we leverage to stabilize the sparse component's training. To this end, we incorporate an always-selected partition and employ a LoRA strategy to maintain a nearly constant parameter count\footnote{LoRA is applied only to the QK projections, while V projections are shared across all partitions, including the always-selected one.}. Leveraging parameter sharing and the sparsity induced by hard classification, SSE enables scaling up the state size of linear attention while keeping both computational and parameter overhead nearly constant.

Empirical analysis reveals that SSE's state partitions exhibit low inter-partition similarity; furthermore, state rows within each partition remain diverse, as illustrated in Figure~\ref{fig:state_sim}(d). Overall, the state exhibits a diagonal pattern, indicating low cross-row similarity and effective information classification. Furthermore, as Figure~\ref{fig:state_entropy} shows, the singular value entropy\footnote{The singular value entropy is computed by $H=-\frac{1}{\log n}\sum_{i=1}^n\frac{\sigma_i^2}{\sum_{j=1}^n\sigma_j^2}\log \frac{\sigma_i^2}{\sum_{j=1}^n\sigma_j^2}$, where $\sigma_i$ (for $i \in \{1,\dots, n\}$) are the singular values of a contextual state matrix. We average $H$ across all attention heads for a specific layer.}~\citep{roy2007effective} of the SSE states is higher than that of GLA~\citep{yang2024gated}. This indicates a more diverse and less compressible state composition, suggesting a more efficient utilization of state capacity in SSE. In addition, we analyze the receptive field of SSE and observe that it is significantly larger than that of GLA, indicating improved long-range information access (see Appendix~\ref{app:exp_details} for details).

\subsection{Efficient Implementations of SSE}
SSE expands the state into multiple partitions, each containing distinct token subsets. During operator execution, we prioritize maintaining parallelism across partitions, avoiding sequential computation. Concurrently, we leverage sparsity to minimize unnecessary computational overhead. This section first describes a naive implementation suitable for variable-length short-sequence settings. Subsequently, we introduce the optimized varlen implementation designed for large-scale and long-sequence scenarios. Finally, we present a straightforward fusion strategy for managing shared partitions. Additional implementation details and pseudocode for both the masking and varlen implementations are provided in Appendix~\ref{app:ops_pseudocode}.

\textbf{Naive implementation via masking.} Variable-length pretraining often results in numerous short sequences within the corpus. This consequently leads to very small per-partition lengths ($\sim L/N$) for SSE. In such scenarios, grouping tokens by their selected partitions and processing them with new cu\_seqlens can incur significant computational overhead. This is because chunk-wise linear attention operators, typically optimized for chunk sizes that are powers of two, become inefficient when $L/N$ falls below the optimal sizes (e.g., 64 or 128), or even below minimum thresholds (e.g., 16). Given these relatively short and variable sequence lengths, we choose not to physically separate tokens. Instead, we handle partitions by increasing parallelism through replication and applying masking, thereby enabling the use of larger chunk sizes.  Specifically, activations are repeated $N$ times, and the top-$k$ indices derived from $\mathbf{e}_t$ are utilized to mask the QKV vectors to zero, ensuring each state partition attends only to its assigned tokens. During operator execution, the additional dimension introduced by replication can be merged into either the batch or head dimension for efficient parallel processing.

\textbf{Efficient implementation via varlen technique.} During the long-context continual training phase, sequence lengths are typically longer and more uniformly distributed. This enables chunk-wise computation to operate only on relevant tokens within each partition, thereby avoiding the redundant computations over masked tokens inherent in the naive implementation. Specifically, we first use the top-$k$ indices derived from $\mathbf{e}_t$ to reorder the QKV vectors, grouping tokens sequentially by partition (from 1 to $N$) within each sample. We then compute a new cu\_seqlens parameter based on these reordered sequences and their corresponding partition assignments. At this point, each resulting subsequence directly corresponds to a specific partition of a given sample. With this reordering and the updated cu\_seqlens, all partitions can be processed in parallel without introducing additional computational overhead.

\textbf{Fusing shared partitions.} Always-selected shared partitions can be fused through concatenation to avoid multiple sequential calls to the linear attention operator. In the naive implementation, shared partitions are concatenated along the head dimension, and no masking is applied to the shared portion. In the varlen implementation, shared partitions are concatenated with the reordered sequence along the sequence dimension, while segment-wise computation remains controlled via the cu\_seqlens parameter. These modifications allow the linear attention operator to be invoked only once.

In the inference setting, the chunk-wise implementation described above applies directly to the prefilling stage. During decoding, we leverage sparse indices to perform recurrent computation solely on the selected partitions, significantly reducing computational overhead.

We evaluate the runtime performance of different SSE attention implementations, including sequential kernel invocation, naive masking-based implementation, and varlen-based implementation, across input sequence lengths from 8 to 256k tokens. As Figure~\ref{fig:ops_speed} illustrates, the masking-based approach offers competitive efficiency for short sequences ($\leq$1k), whereas the varlen implementation is more efficient for longer sequences. Furthermore, the varlen technique enables SSE to exhibit favorable scalability with respect to state size $N$, maintaining nearly constant runtime when $K$ is fixed.

\begin{figure}
    \centering
    \includegraphics[width=\textwidth]{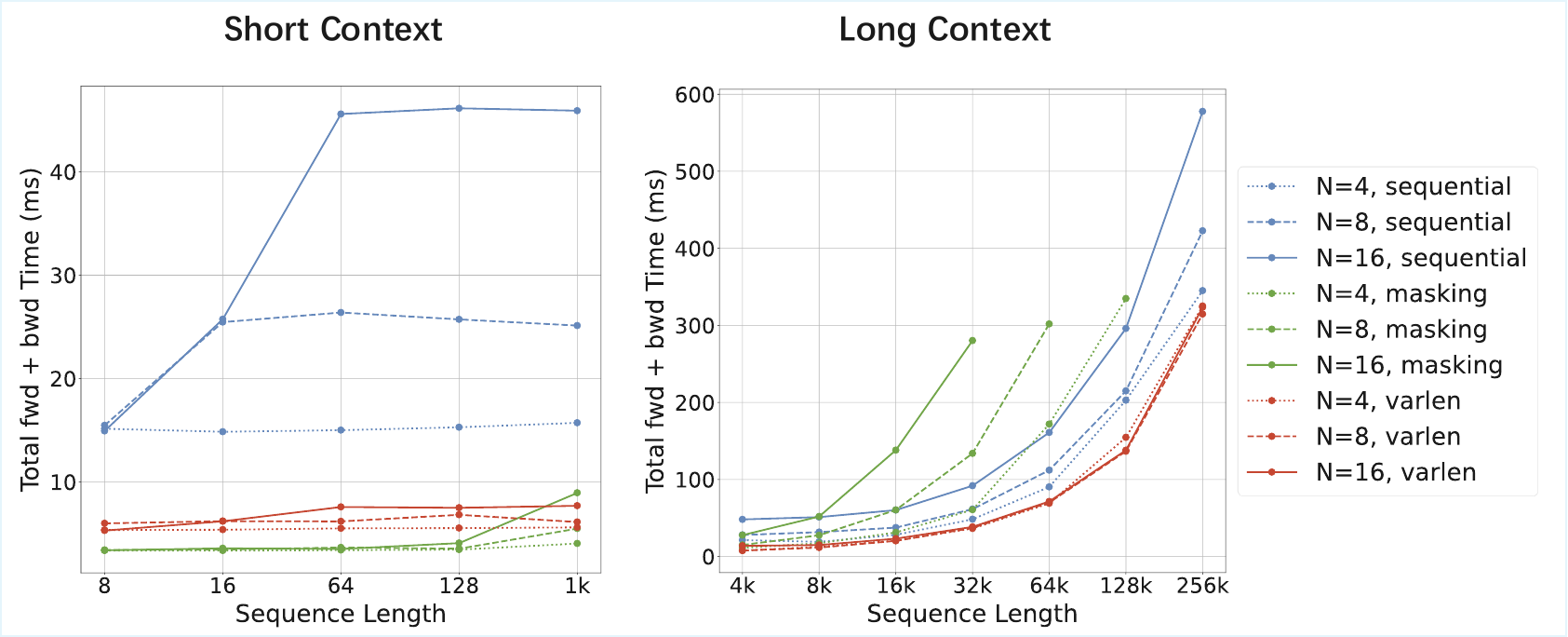}
\caption{\textbf{Speed comparison of different SSE implementations.} The number of selected partitions is fixed at $K = 1$, while the total number of partitions $N$ is varied. We set attention heads to 8, head dimension to 128, and cu\_seqlens to $[0, L/2, L]$. All evaluations are conducted on an A100 GPU.
}
    \label{fig:ops_speed}
\end{figure}

\section{Experiments}
We conduct a comprehensive evaluation of SSE across a wide range of benchmarks, spanning language modeling, in-context retrieval, and reasoning. Across these tasks, SSE consistently delivers stronger performance than existing linear attention models, particularly in in-context retrieval settings. Furthermore, when integrated into a hybrid architecture, SSE achieves overall performance on par with Transformer models. Our evaluation covers the entire training pipeline, including pretraining, long-context extension, distillation, reinforcement learning, and up-training from pretrained Transformer checkpoints. Notably, the hybrid variant of SSE demonstrates reasoning capabilities comparable to state-of-the-art Transformers when trained under the same regimes.

\textbf{Models.} We build all models on the MHA-SwiGLU architecture~\citep{touvron2023llama}, varying only the attention mechanism to ensure fair comparisons. We compare SSE against standard Transformer (using softmax attention) and several representative linear attention baselines, such as GLA~\citep{yang2024gated}, GDN~\citep{yang2024gdn}, and MoM~\citep{du2025mom}. Both SSE and MoM adopt GLA-style state transition mechanisms, with $n$ denoting the state expansion factor (i.e., the number of state partitions) and $k$ representing the top-$k$ hard selection size per token. The sparsity ratio is defined as $k/n$. To ensure a fair comparison, all models exclude convolutional layers. Additionally, we also evaluate layer-wise hybrid architectures, constructed by adding one softmax attention layer after every five linear attention layers\footnote{In our 2B setting, this results in 3 softmax attention layers out of 18 total layers.}. These models are denoted by the "-H" suffix (e.g., SSE-H).

\textbf{Training Setup.} We conducted our primary experiments with two model scales: 600M and 2B parameters\footnote{The non-embedding parameters are 300M and 1.3B, respectively.}. During pretraining, all models are trained with AdamW, using a maximum sequence length of 8192 and a total batch size of 4M tokens. The learning rate schedule includes a linear warm-up phase followed by a constant learning rate of 6e-4. We apply a weight decay of 0.1 and a gradient clipping of 1.0. Following pre-training, we perform context length extension up to 32k tokens, using cosine learning rate decay. Subsequently, we perform a reasoning-oriented data distillation stage using a fine-tuning set of approximately 80k examples, trained for 5 epochs. Finally, we apply reinforcement learning with GRPO~\citep{shao2024deepseekmath} for 230 steps. During RL training, we sample 8 responses per prompt with a generation limit of 32k tokens.

\begin{table*}[t]
\centering
\small
\setlength{\tabcolsep}{4pt}
\centering
\caption{\textbf{Zero-shot language modeling and recall performance of pretrained models.} For SSE and MoM, $n$ and $k$ (e.g., n4k1) indicate the expansion ratio and sparsity ratio, respectively. Both SSE and MoM employ diagonal gating, similar to GLA. MoM is evaluated without token truncation to establish a stronger baseline.}
\label{tab:language_modeling}{
\resizebox{\textwidth}{!}{
\begin{tabular}{l|c|ccccccc|cccc}
\toprule
\multirow{2}{*}{\textbf{Model}} & & \multicolumn{7}{c}{\textbf{CommonSense Reasoning}} & \multicolumn{4}{c}{\textbf{Real-world Recall}}\\
\cmidrule(lr){3-9} \cmidrule(lr){10-13} 
& Wiki.$\downarrow$ & PIQA & Hella. & Wino. & ARC-e & ARC-c & SIQA & \textbf{Avg.} & FDA & SWDE & SQuAD & \textbf{Avg.} \\
\midrule
\multicolumn{13}{c}{600M params, 15B tokens} \\
\midrule
\rowcolor{gray!12}
Transformer & 33.47 & 61.92 & 31.60 & 51.38 & 48.02 & 23.12 & 37.31 & 42.22 & 74.50 & 60.67 & 32.67 & 55.95 \\
GLA & 43.97 & 61.21 & 29.61 & \underline{51.46} & 47.01 & 23.29 & 36.59 & 41.53 & 9.44 & 23.40 & 23.06 & 18.63 \\
GDN & \textbf{35.26} & \underline{62.79} & \textbf{32.16} & \textbf{51.85} & \textbf{50.46} & 24.23 & 36.80 & \textbf{43.05} & 14.79 & \underline{33.48} & 26.24 & 24.84 \\
\textbf{SSE-n4k1} & 35.52 & 61.59 &\underline{31.71}& \underline{51.62}& 49.45 &\underline{24.74}& \textbf{38.33}& 42.91& \textbf{33.67} &33.30&  \underline{26.51}& \underline{31.16} \\
\textbf{SSE-n4k2} & \underline{35.33} & \textbf{62.95}& 31.64 &50.43 &\underline{49.66} &\textbf{25.34} &\underline{38.28} &\textbf{43.05} &  \underline{30.85}& \textbf{36.90}&  \textbf{27.82}& \textbf{31.86} \\
\midrule
\multicolumn{13}{c}{800M (600M active) params, 15B tokens} \\
\midrule
MoM-n4k1 & 36.05 & 62.57 & 31.64 & 51.70 & 48.53 & 23.46 & 37.92 & 42.64 & 18.97 & 36.36 & 27.75 & 27.69 \\
MoM-n4k2 & 35.11 & 63.87 & 32.47 & 49.64 & 49.54 & 23.81 & 37.36 & 42.78 & 21.05 & 37.08 & 26.68 & 28.27 \\
\midrule
\multicolumn{13}{c}{2B params, 100B tokens} \\
\midrule
\rowcolor{gray!12}
Transformer & 16.46 & 71.82& 52.70&  58.80&  63.76& 31.31& 42.89& 53.55 & 86.48 & 82.99 & 49.53 & 73.00 \\
GLA & 21.30 & 68.44& 42.82& 54.38& 59.43& 28.24& 41.45& 49.13& 51.27 & 59.86 & 36.73 & 49.29 \\
GDN & \underline{17.08} & \textbf{72.25}& \underline{53.20}&  \underline{57.70} & \underline{65.61} &\underline{32.68} &\textbf{42.89} &\underline{54.05} & \underline{54.63} & \underline{66.88} & \underline{39.88} & \underline{53.79} \\
\textbf{SSE-n4k1} & \textbf{16.92} & \underline{71.98}& \textbf{53.62} &\textbf{60.14} &\textbf{66.67}& \textbf{32.94}& \underline{42.07}& \textbf{54.57} & \textbf{71.23} & \textbf{69.94} & \textbf{43.20} & \textbf{61.46}\\
\midrule
GLA-H & 17.27 & 70.95 &51.58 &57.30 & 65.03 &31.31 &42.73 &53.15 & 78.22 & 78.49 & 45.84 & 67.52 \\
\textbf{SSE-H-n4k1} & \textbf{16.47} & \textbf{71.93} & \textbf{53.65} & \textbf{60.38} & \textbf{65.28} & \textbf{31.91} & \textbf{43.71} &\textbf{54.48}& \textbf{84.48} & \textbf{80.65} & \textbf{47.49} & \textbf{70.87} \\
\bottomrule
\end{tabular}}
}
\end{table*}

\subsection{Language Modeling and Retrieval}

\textbf{Language Modeling.} We conduct small-scale pretraining to compare different attention mechanisms, using 600M and 2B models trained on 15B and 100B tokens, respectively. Subsequently, we evaluate these models on zero-shot commonsense reasoning performance. As shown in Table~\ref{tab:language_modeling}, SSE-n4 demonstrates strong results among advanced linear attention models, outperforming Transformers and matching MoM despite using fewer parameters (600M vs. 800M). Furthermore, in a hybrid setting, SSE-H also surpasses both GLA-H and Transformers at the 2B scale, reaching state-of-the-art performance.

\textbf{In-context Retrieval.} Retrieval is a long-standing bottleneck for linear attention models due to their limited capacity for precise token-level recall. SSE addresses this challenge through a combination of sparse state expansion and row-sparse updates, enabling more expressive and efficient memory access. As Table~\ref{tab:language_modeling} illustrates, SSE consistently outperforms other linear attention models on real-world retrieval-intensive tasks, narrowing the gap with the Transformer baseline. At both 600M and 2B scales, SSE-n4k1 significantly improves recall accuracy over GLA and GDN on all tasks, despite GDN’s use of an explicit erasure mechanism. While pure linear variants still lag behind Transformer, the hybrid SSE-H further boosts retrieval performance, surpassing GLA-H and closely approaching Transformer-level accuracy.

We also benchmark our 2B models on synthetic S-NIAH tasks within the RULER~\citep{hsiehruler} suite. As shown in Table~\ref{tab:ruler}, SSE demonstrates superior retrieval capabilities under controlled conditions, further validating the benefits of its architectural design.

\textbf{Data Scaling and Length Extension.} We further scale both the 2B SSE-H and Transformer baseline models to 2T pretraining tokens, followed by a 32k context extension phase using an additional 250B tokens. Subsequently, we evaluate these models across a broad set of benchmarks covering knowledge-intensive, reasoning-oriented, and retrieval-based tasks. As presented in Table~\ref{tab:d8_all}, SSE-H remains competitive with the Transformer baseline across all task categories, while achieving higher average accuracy overall. In particular, SSE-H outperforms the Transformer on several challenging benchmarks, including MMLU, MMLU-Pro, and C-Eval, which are widely used to assess general knowledge and multilingual reasoning ability. We also report results on six single- and multi-needle tasks up to 32k from RULER, where SSE-H shows strong long-context retrieval performance (Appendix~\ref{app:exp_details}, Table~\ref{tab:ruler_32k}). These results highlight the robustness and scalability of our hybrid architecture under large-scale training and long-context settings.

\begin{table*}[t]
\centering
\small
\centering
\caption{\textbf{Performance comparison on Single-NIAH tasks in RULER.} All models have 2B parameters and are trained on 100B tokens using 8k variable-length sequences. Results are reported in a zero-shot setting.}
\label{tab:ruler}{
\begin{tabular}{l|ccc|ccc|ccc}
\toprule
\multirow{2}{*}{\textbf{Model}} & \multicolumn{3}{c}{\textbf{S-NIAH-1}} & \multicolumn{3}{c}{\textbf{S-NIAH-2}} & \multicolumn{3}{c}{\textbf{S-NIAH-3}} \\
\cmidrule(lr){2-4} \cmidrule(lr){5-7} \cmidrule(lr){8-10} 
& 2K & 4K & 8K & 2K & 4K & 8K & 2K & 4K & 8K \\
\midrule
Transformer & 100.0 & 100.0 & 100.0 & 100.0 & 100.0 & 100.0 & 100.0 & 100.0 & 100.0 \\
\midrule
GLA & 100.0 & 100.0 & 87.6 & 100.0 & 81.8 & 23.2 & 88.0 & 62.2 & 16.2\\
GDN & 100.0 & 100.0 & 100.0 & 98.8 & 62.4 & 8.2 & 96.2 & 63.8 & 9.0 \\
\textbf{SSE-n4k1} & 100.0 & 100.0 & 100.0 & 100.0 & 99.6 & 85.2 & 100.0 & 88.2 & 62.2 \\
\midrule
GLA-H & 100.0 & 100.0 & 100.0 & 100.0 & 100.0 & 100.0 & 98.8 & 96.4 & 70.8 \\
\textbf{SSE-H-n4k1} & 100.0 & 100.0 & 100.0 & 100.0 & 100.0 & 100.0 & 100.0 & 99.4 & 97.4 \\
\bottomrule
\end{tabular}}
\end{table*}

\begin{table*}[h]
\centering
\small
\setlength{\tabcolsep}{4pt}
\centering
\caption{\textbf{Benchmarking of 2B SSE-H and Transformer baseline after long-context extension.}}
\label{tab:d8_all}{
\resizebox{\textwidth}{!}{
\begin{tabular}{l|cccccccccc|c}
\toprule
      \textbf{Model} & MMLU & MMLU-Pro & C-Eval & AGIEval & TrviaQA & BBH & SWDE& SQuAD & Drop & GSM8K & \textbf{Avg.} \\
\midrule
Transformer &          52.6 &          24.2 &          55.9 & 38.9 &          21.8 &          38.8 & \textbf{85.5} & \textbf{52.0} &          \textbf{30.9} & \textbf{50.6} &          45.1 \\
\textbf{SSE-H-n4k1} & \textbf{54.5} & \textbf{26.1} & \textbf{59.7} &         \textbf{39.6} & \textbf{22.7} & \textbf{39.2} &          84.0 &          51.1 & 30.1 &          49.2 & \textbf{45.6} \\
\bottomrule
\end{tabular}}
}
\end{table*}

\textbf{Converting Pre-trained Transformers to Hybrid Models.} We further explore a conversion strategy~\citep{kasai2021finetuning,zhang2024gated} that replaces a subset of softmax attention layers in a pretrained Transformer with linear attention layers, followed by continued training on a limited dataset. Specifically, we apply this method to a 2B Transformer during its 128k long-context training stage, using approximately 100B tokens. We evaluate the resulting models on three 32k-context retrieval tasks. As shown in Table~\ref{tab:1b_convert}, the converted SSE-H model achieves substantially better retrieval consistency than GLA-H, and further narrows the performance gap with the original Transformer.

\begin{table*}[h]
\centering
\small
\setlength{\tabcolsep}{4pt}
\centering
\caption{\textbf{Evaluation of long-range retrieval (32k) under the conversion paradigm.}}
\label{tab:1b_convert}{
\begin{tabular}{l|c|ccc|c}
\toprule
\textbf{Model} & Params. & Quest & Qampari & Table Query & \textbf{Avg.} \\
\midrule
Transformer &2B &49.1& 71.5& 60.6 & 60.4\\
GLA-H &2B &25.6& 56.6& 48.2 & 43.5\\
\textbf{SSE-H-n4k1} &2B &35.6& 70.6&50.0& 52.1\\
\bottomrule
\end{tabular}}
\end{table*}

\subsection{Reasoning Ability}

We further evaluate whether hybrid linear attention models like SSE-H can serve as effective reasoning models when trained under standard pipelines that include supervised distillation and reinforcement learning. Reasoning presents a particular challenge for efficient architectures, as it often requires multi-step computation, symbolic manipulation, and precise memory access—all traditionally dominated by full-attention Transformers.

To assess this capability, we follow common practice and evaluate on a suite of math-focused benchmarks, including MATH500, AIME24, AIME25, AMC23, and OlympiadBench. For AIME24 and AIME25, we repeat the evaluation set 32 times and report average accuracy (avg@32). Inference uses a temperature of 1.0 and top-p of 0.7. For comparison, we include a range of strong open-source small-scale reasoning models, covering both softmax and hybrid linear attention designs.

As shown in Table~\ref{tab:reasoning}, our 2B SSE-H-n4k1 achieves substantial gains over similarly sized open-source softmax Transformers—improving AIME24 accuracy from 48.3 to 64.5 (+16.2), and AIME25 from 36.8 to 50.2 (+13.4). Its performance also approaches that of much larger models like DeepSeek-R1-Distill-Qwen-7B~\citep{guo2025deepseek}, underscoring the strength of hybrid models for high-level reasoning.

Importantly, when trained with exactly the same pipeline as the Transformer baseline, SSE-H achieves comparable results, confirming that hybrid linear attention can match softmax attention in mathematical reasoning ability. To further assess scalability, we also convert a pretrained 12B Transformer to an SSE-H variant, which achieves similar reasoning accuracy to its softmax-attention counterpart (see Appendix~\ref{app:exp_details}).

These results demonstrate that SSE-H is not only competitive with state-of-the-art small reasoning models but also shows significant promise as a scalable and efficient architecture for reasoning tasks and reinforcement learning applications.

\begin{table*}[t]
\centering
\small
\setlength{\tabcolsep}{4pt}
\centering
\caption{\textbf{Reasoning ability of the SSE-H model.} Average accuracy (avg@32) for AIME24 and AIME25 is reported, based on 32 repetitions of the evaluation set. Inference is conducted with a temperature of 1.0 and a top-p of 0.7.}
\label{tab:reasoning}{
\begin{tabular}{l|ccccc}
\toprule
\textbf{Model} & AIME24 & AIME25 & MATH500 & OlympiadBench & AMC23 \\
\midrule
DeepScaleR-1.5B-Preview~\citep{deepscaler2025} & 43.1 & - & 87.8 & 50.0 & 73.6 \\
Qwen3-1.7B (Thinking)~\citep{yang2025qwen3} & 48.3 & 36.8 & 93.4 & - & - \\
DeepSeek-R1-Distill-Qwen-1.5B~\citep{guo2025deepseek} & 28.9 & 23.5 & 83.9 & 43.3 & 62.9 \\
DeepSeek-R1-Distill-Qwen-7B~\citep{guo2025deepseek} & 55.5 & 39.2 & 92.8 & - & - \\
\midrule
M1-3B~\citep{wang2025m1} & 23.0 & 22.0 & 81.7 & 43.6 & 56.0 \\
PromptCoT-Mamba-7B~\citep{zhao2025scaling} & 35.2 & 24.6 & 84.6 & 50.7 & - \\
\midrule
Transformer-2B (Ours) & 64.1 & 52.8 & 93.0 & 83.3 & 92.0 \\
\textbf{SSE-H-n4k1-2B} (Ours)  & 64.5 & 50.2 & 92.1 & 85.7 & 91.4 \\
\bottomrule
\end{tabular}}
\end{table*}

\subsection{Ablation Study}

\textbf{Effectiveness of State Size Scaling in SSE Recall.}  
We evaluate this by conducting experiments on 2B models trained on 20B tokens, varying both the number of partitions ($n$) and the top-$k$ hard selection size ($k$), and measuring the corresponding changes in recall performance. The results are presented in Figure~\ref{fig:state_scale}.
Under the sparse update framework, state expansion in SSE exhibits promising scalability without increasing parameter count:

(1) When the sparsity ratio $k/n$ is held constant and the total number of partitions $n$ increases, recall performance improves approximately proportionally with $n$.
(2) For a fixed $n$, increasing the sparsity ratio $k/n$ within a moderate range (e.g., up to 50\%) also leads to recall gains. Notably, the benefit saturates beyond a certain threshold. In particular, when $k/n = 1$, SSE effectively reduces to vanilla linear attention.

While demonstrating promising scalability, increasing the sparsity ratio and partition number incurs additional computational costs and inference latency. Consequently, we employed the $n4k1$ setting for the majority of our empirical analysis, particularly for the hybrid architecture. This indicates that SSE's performance still exhibits substantial potential for further improvement.




\begin{figure}[t]
  \centering

  \begin{minipage}[b]{0.53\textwidth}
    \centering
    \includegraphics[width=\linewidth]{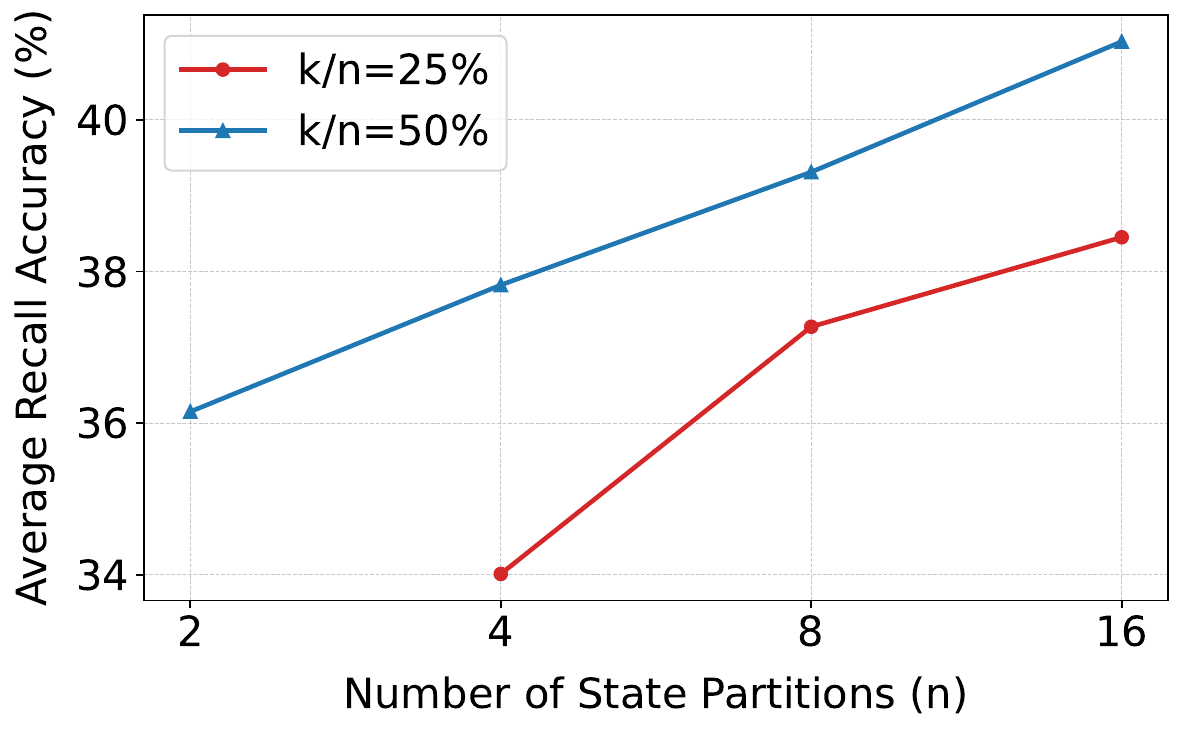}
    \captionof{figure}{\textbf{Scalability of SSE with respect to state capacity.} $n$ denotes the number of partitions; $k$ is the top-$k$ hard selection size.}
    \label{fig:state_scale}
  \end{minipage}
  \hfill
  \begin{minipage}[b]{0.43\textwidth}
    \centering
    \includegraphics[width=\linewidth]{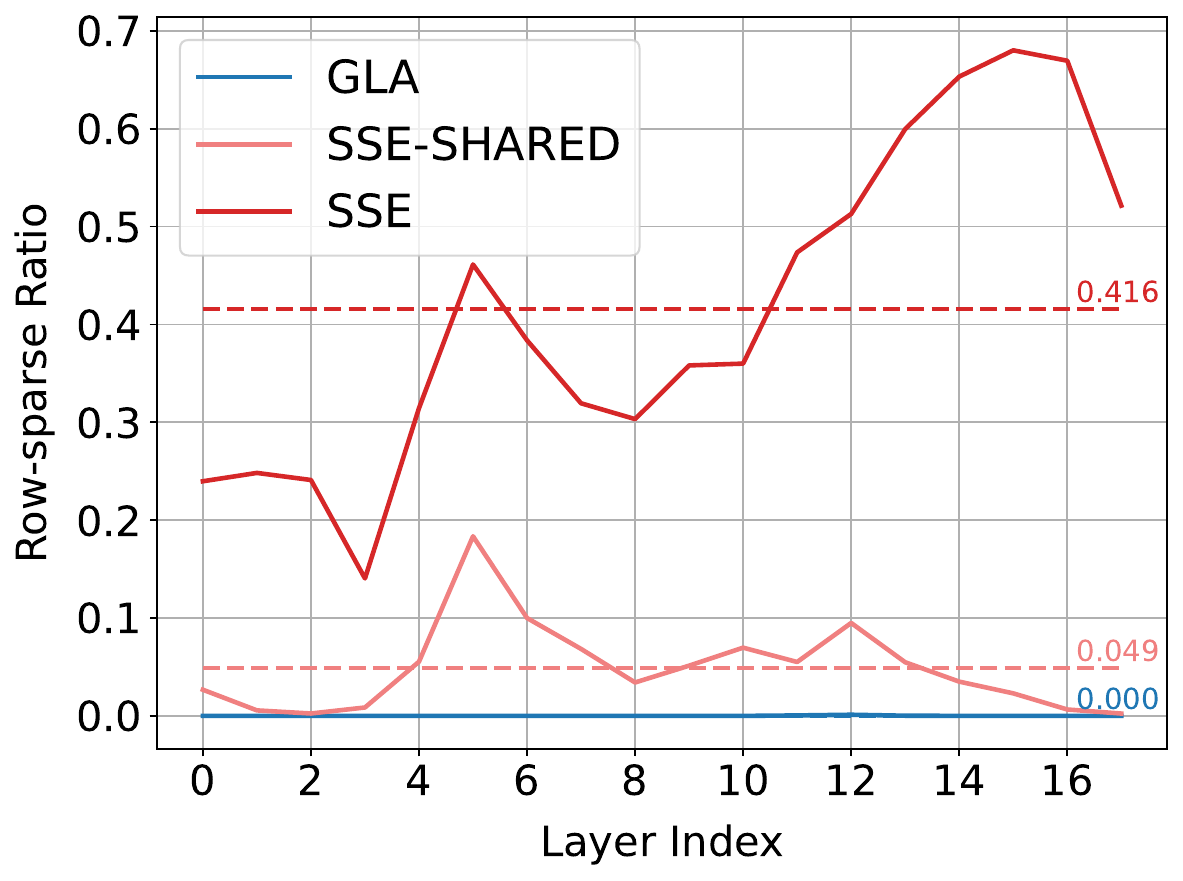}
    \captionof{figure}{\textbf{Row sparsity ratio of state updating.} The row sparsity ratio is measured by the proportion of channels in $\mathbf{k}_t$ with values below $1\text{e-}5$.}
    \label{fig:row_sparsity}
  \end{minipage}

\end{figure}

\textbf{Effectiveness of Softmax-based Row Selection.}  
We conduct an ablation study on 600M models trained with 15B tokens to evaluate the effectiveness of the proposed softmax feature map for row-sparse updates under state expansion. As Table~\ref{tab:feature_map_ablation} shows, using softmax as the feature map for $\mathbf{k}_t$ yields significantly higher recall compared to the commonly used SiLU map. This is consistent with our earlier analysis based on information classification.

We further examine the sparsity of $\mathbf{k}_t$ vectors in 2B GLA and SSE-n4k1 models, using $1\text{e-}5$ as the threshold. As shown in Figure~\ref{fig:row_sparsity}, although SiLU has a value range near zero, the trained GLA model does not exhibit any sparsity. In contrast, SSE with the softmax classifier induces sparsity in the selected partitions, with ratios of 5\% and 42\%. Interestingly, the model learns to write densely to the shared partition while maintaining highly sparse updates in the non-shared partitions.

\begin{table*}[h]
\centering
\small
\setlength{\tabcolsep}{4pt}
\centering
\caption{\textbf{Ablation study on the effect of softmax-based row selection.}}
\label{tab:feature_map_ablation}{
\begin{tabular}{l|cc|ccc|c}
\toprule
\textbf{Models}  &Params.  &Tokens  &FDA  &SWDE & SQuAD & \textbf{Avg.} \\
\midrule
SSE-n4k1-k.silu &600M &15B	& 20.87 &26.91& 24.90&  24.23\\
SSE-n4k1-k.softmax &600M &15B	& 33.67 &33.30&  26.51& 31.16 (+6.93)\\
\bottomrule
\end{tabular}}
\end{table*}

\textbf{Effectiveness of Write-Read Gating.}
Within SSE’s hard partition selection framework, we introduce a write-read gate that operates on both the state input and output (i.e., KV and Q) simultaneously. We conduct a comprehensive ablation study on 2B models trained with 100B tokens, comparing four variants: (1) no gating, (2) input-only write gate, (3) output-only read gate, and (4) our proposed write-read gate. As Table~\ref{tab:gate_ablation} illustrates, the write-read gate consistently outperforms all alternatives on both recall and commonsense reasoning tasks, demonstrating the effectiveness of partition selection design in SSE.

\begin{table*}[h]
\centering
\small
\setlength{\tabcolsep}{4pt}
\centering
\caption{\textbf{Ablation study on write-read gating.} \textbf{Models:} SSE with write–read gating (applied to both input and output) compared against three variants: no gate, write-only gate, and read-only gate. \checkmark indicates the gate is applied, and \ding{55} indicates it is not. \textbf{Metrics:} CSR-Avg. is the average of commonsense reasoning tasks.}
\label{tab:gate_ablation}{
\begin{tabular}{l|cc|cc|c|ccc|c|c}
\toprule
\textbf{Models}  &Write& Read &Params.  &Tokens &Wiki.$\downarrow$  &FDA  &SWDE & SQuAD & \textbf{Recall-Avg.}  &\textbf{CSR-Avg.} \\
\midrule
SSE-n4k1&\checkmark&\checkmark &2B &100B &18.44	& 64.07 &65.17& 40.65 &56.63 &53.63\\
\quad no gate&\ding{55}&\ding{55} &2B &100B &18.63 & 51.91 &63.91 &39.81 &51.87 (-4.76) &53.01 (-0.62)\\
\quad write gate&\checkmark&\ding{55} &2B &100B &18.65	& 54.54& 61.12& 37.10&  50.92 (-5.71) &52.82 (-0.81)\\
\quad read gate&\ding{55}&\checkmark &2B &100B &18.62	& 57.26& 64.45& 40.18& 53.96 (-2.67) &52.91 (-0.72)\\
\bottomrule
\end{tabular}}
\end{table*}

\textbf{Effectiveness of Shared Attention Parameters.}
SSE shares attention parameters (QKV and gate projection) across partitions, while applying LoRA to decouple QK projections for the always-selected partition. For comparison, we train a variant where each of the $N+1$ partitions has its own attention parameters, denoted as "wo/ shared-params". As shown in Table~\ref{tab:params_ablation}, despite nearly doubling the non-embedding parameters, this variant shows lower recall and only marginal gains on commonsense reasoning. These results support the key assumption of SSE: overcoming the limitations of linear attention depends on state capacity rather than parameter count. Effective parameter sharing not only enables the model to learn segmented clustering of input information, but may also facilitate more stable optimization.

\begin{table*}[h]
\centering
\small
\setlength{\tabcolsep}{4pt}
\centering
\caption{\textbf{Ablation study on shared attention parameters.} \textbf{Models:} SSE shares attention parameters across partitions, while "wo/ shared-params" denotes a variant with separate parameters for each partition.  \textbf{Metrics:} CSR-Avg. is the average of commonsense reasoning tasks.}
\label{tab:params_ablation}{
\resizebox{\textwidth}{!}{
\begin{tabular}{l|cc|c|ccc|c|c}
\toprule
\textbf{Models}  &Params. (non-embed)  &Tokens &Wiki.$\downarrow$  &FDA  &SWDE & SQuAD & \textbf{Recall-Avg.}  &\textbf{CSR-Avg.} \\
\midrule
SSE-n4k1 &600M (300M) &15B &35.52	& 33.67 &33.30&  26.51& 31.16 &42.91\\
\quad wo/ shared-params &890M (580M) &15B &34.94	& 16.52& 36.45& 24.87& 25.95 (-5.21) &42.99 (+0.08)\\
\bottomrule
\end{tabular}}
}
\end{table*}

\section{Limitations and Future Work}
While our experiments demonstrate the promising scalability of SSE with GLA-style transition mechanisms, we have not yet conducted comprehensive evaluations with delta-rule baselines (e.g., GDN). Moreover, more sophisticated classification functions, such as state-aware classifiers, could be effectively combined with SSE-GDN, offering a promising direction for future work.

Concurrently, we observe that non-hybrid SSE still exhibits a notable retrieval gap compared to well-trained Transformer models. Inspired by SSE's effective state size scalability, we consider further investigation into SSE's state scaling law a potential solution. Given the high search cost of partition count as a hyperparameter, we plan to utilize an up-training scheme to expand state size based on existing checkpoints. Thanks to our parameter sharing strategy, SSE can smoothly transition between different $N,K$ configurations, quickly recovering general ability without requiring extensive retraining or distillation. We consider SSE's state up-training a key direction for future work.

We acknowledge that SSE's current context-only hard-selection mechanism may not be optimal and warrants broader exploration. The computation of $\mathbf{e}_t$ provides a flexible interface for incorporating inductive biases into the model. A compelling future direction involves incorporating dependencies on positional and historical information, leading to a more general form such as $\mathbf{e}_t = g(t, \mathbf{x}_t, \mathbf{S}_{t-1})$.

\section{Related Work}
\textbf{Linear Attention.} Linear attention~\citep{katharopoulos2020transformers} replaces the softmax kernel of Transformer attention with a dot-product of feature maps, thereby achieving linear complexity and more efficient long-context inference. Early research in this area~\citep{choromanskirethinking,peng2021random,choromanskihybrid,zhanghedgehog} primarily focused on designing specialized non-negative feature maps to approximate the behavior of the softmax kernel. Subsequent works~\citep{sun2023retentive,qin2023transnormerllm} further simplified linear attention, removing unnecessary normalizer terms and other specialized feature maps. Concurrent with the evolution of linear attention, other generalized linear recurrent mechanisms emerged, such as State Space Models (SSMs)~\cite{guefficiently,gu2022parameterization,fu2022hungry,gumamba,dao2024transformers} and linearized RNNs~\citep{qin2023hierarchically,peng2023rwkv,orvieto2023resurrecting,beckxlstm}. Through continuous development, these mechanisms have converged towards a unified decay and outer-product update formulation~\citep{chou2024metala}, similar to the unnormalized linear attention discussed in our work. 

Recent findings~\citep{yang2024gated,gumamba} highlight the benefit of introducing input-dependency into state decay, enabling more adaptive memory control through gating mechanisms. More recent trends involve densifying the state transition matrix, exemplified by delta-rule-based attentions~\citep{yangparallelizing,yang2024gdn,peng2025rwkv,siems2025deltaproduct}. While sophisticated structures, such as various TTT variants~\citep{sun2024learning,behrouz2024titans,wang2025test,behrouz2025s} designed based on online learning, offer advanced state updates, their nonlinear dependencies often introduce trade-offs in implementation efficiency. In contrast to these works that primarily focus on the state transition matrix, our work re-examines how new information is more effectively incorporated into the state. We achieve this by innovatively connecting the state update mechanism to an information classification process.

\textbf{State Expansion.} Recent literature~\citep{arora2024zoology,arora2024simple,jelassi2024repeat,akyurek2024context,waleffe2024empirical} consistently highlights the challenges posed by limited state size in linear attention models, particularly evident in retrieval tasks. Consequently, state expansion has emerged as a seemingly indispensable direction. Beyond directly expanding the state, Log-Linear Attention~\citep{guo2025log} adopts a Fenwick tree structure to enable logarithmically growing state sizes. Our work bears the closest resemblance to Mixture-of-Memories (MoM)~\citep{du2025mom}, which models the state using a Mixture-of-Experts (MoE) approach and shares a similar sparse spirit with SSE. However, SSE designs a sparse state expansion scheme specifically from an information classification perspective, explicitly decouples state capacity from parameter count, and introduces improved feature mapping and gating mechanisms. Furthermore, we provide a comprehensive theoretical framework demonstrating that SSE's state achieves more effective information storage. Our extensive experimental pipeline, spanning pretraining to RL, robustly confirms the proposed model's stability.

\textbf{Hybrid Models and Reasoning.} Beyond solely expanding the state of linear attention, a common compromise involves hybridizing some softmax attention layers. This approach sacrifices a degree of computational efficiency in exchange for significant improvements in retrieval performance. Recent works have explored both inter-layer~\citep{lieber2024jamba,glorioso2024zamba,yangparallelizing,li2025minimax,wang2025systematic} and intra-layer~\citep{dong2024hymba,li2025transmamba,behrouz2024titans} hybrid paradigms. Notably, such hybrid architectures have recently demonstrated promising overall performance, even at large scales~\citep{li2025minimax,chen2025minimax,liu2025hunyuan}. Concurrently, the reasoning capabilities of pure linear attention models remain relatively underexplored, largely relying on hybrid architectures~\citep{wang2025m1,liu2025hunyuan,chen2025minimax,zhao2025scaling}. Our work is the first to demonstrate that, under identical pre-training paradigms, hybrid linear models can achieve reasoning performance on par with Transformer models, thereby enjoying the advantages of test-time scaling. Through fine-tuning and advanced RL training, we report that our 2B SSE-hybrid model exhibits exceptional mathematical reasoning ability, achieving state-of-the-art performance at its scale.

\section{Conclusion}
In this work, we enhance the capacity of existing linear attention models in handling long contexts by proposing two key innovations: row-sparse state updates and Sparse State Expansion (SSE). Our framework conceptualizes state updates as an information classification process and enables efficient expansion of contextual states into multiple partitions with controlled parameter count. This design facilitates both scalable memory capacity and more discriminative state representations. We demonstrate that SSE and its hybrid variant (SSE-H) achieve superior performance across language modeling, in-context retrieval, and mathematical reasoning tasks. Notably, our 2B SSE-H model sets state-of-the-art reasoning performance among small-scale models on math-focused benchmarks, significantly surpassing comparably sized Transformers. Together with its favorable scalability, these results position SSE as a promising and efficient architecture for high-fidelity long-context modeling.
\clearpage
\section{Contributions}
\label{sec:Contributions}

\subsection*{Project Lead}
Zheng Li$^{1}$

\makeatletter
\renewcommand{\thefootnote}{$\dagger$}
\makeatother

\subsection*{Core Contributors\footnotemark[1]}
\footnotetext[1]{Equal Contribution.}
\makeatletter
\renewcommand{\thefootnote}{$\ddagger$}
\makeatother
Yuqi Pan$^{1,2}$\footnotemark[2], Yongqi An$^{1,2}$\footnotemark[2], Zheng Li$^{1}$
\footnotetext[2]{Work done during internship at ByteDance Seed.}

\subsection*{Contributors}
Yuhong Chou$^{3}$, Ruijie Zhu$^{1,4}$ 

\subsection*{Supervisors}
Xiaohui Wang$^{1}$, Mingxuan Wang$^{1}$, Jinqiao Wang$^{2}$, Guoqi Li$^{2}$

\subsection*{Affiliation}
$^{1}$ByteDance Seed\\
$^{2}$Institute of Automation, Chinese Academy of Sciences\\
$^{3}$The Hong Kong Polytechnic University\\
$^{4}$UC Santa Cruz

\makeatletter
\renewcommand{\thefootnote}{\arabic{footnote}}  
\makeatother

\clearpage

\bibliographystyle{plainnat}
\bibliography{main}

\clearpage

\beginappendix

\section{Analyzing State Updating as Information Classification}\label{app:theory}
In this section, we present theoretical analyses and key insights related to the information classification framework for state representations. Our discussion primarily focuses on the case of a linear classifier, namely:
\begin{align}
\mathbf{k}_t&=f(\mathbf{x}_t, \mathbf{W}_k) = \mathbf{x}_t\mathbf{W}_k,\\
    \mathbf{S}_t &= \mathbf{\Lambda}_t \mathbf{S}_{t-1}+\mathbf{k}_t^{\top}\mathbf{v}_t.
\end{align}

\textbf{The contextual state updates in linear attention act as an information classification process.} By conceptualizing $\mathbf{k}_t = \mathbf{x}_t \mathbf{W}_k$ as a classification function, we can demonstrate that information assigned to the same state row exhibits similar features, effectively making the state $\mathbf{S}_t$ function as an information classifier. In the linear classifier setting, information similarity is measured by the inner product. The classification result for class $i$ is defined as $C^i=\{\mathbf{x}_s|\mathbf{x}_s \mathbf{W}_k^i > k_{th}^i\}$, where $k_{\text{th}}^i > 0$ is a class-specific threshold. Proposition~\ref{prop:class} shows that information assigned to the same state row exhibits large mutual inner products.

\begin{proposition}\label{prop:class}
    For inputs $\mathbf{x}_s$ satisfying $||\mathbf{x}_s||^2=d$, and given the classification rule $C^i=\{\mathbf{x}_s|\mathbf{x}_s \mathbf{W}_k^i > k_{th}^i\}$, inputs belonging to the same class (row) $C^i$ satisfy $\mathbf{x}_r\mathbf{x}_s^{\top}>d\cos{2\theta^i}$, where $\theta^i=\arccos(\frac{k_{th}^i}{\sqrt{d}||\mathbf{W}_k^i||})$, $s,r\in\{1,\dots,t\}$, and vice versa.
\end{proposition}
\begin{proof}
    The classification rule implies inputs with high acceptance intensity $\mathbf{k}_i=\mathbf{x}_s \mathbf{W}_k^i$ belong to the $i$-th class (row), and the classification threshold $k_{th}^i>0$.

    If $\mathbf{x}_s\in C^i$, then 
    \begin{align}
        \mathbf{x}_s \mathbf{W}_k^i &= \sqrt{d}||\mathbf{W}_k^i||\cos\theta > k_{th}^i,\\
        \Rightarrow \cos\theta &> \frac{k_{th}^i}{\sqrt{d}||\mathbf{W}_k^i||}=\cos\theta^i, 
    \end{align}
    which means all inputs from class $i$ have an angle with $\mathbf{W}_k^i$ less than $\theta^i$, where $k_{th}^i\in (0, \sqrt{d}||\mathbf{W}_k^i||)$ and $\theta^i\in(0, \frac{\pi}{2})$. This is equivalent to saying that the angle between inputs of the same class $i$ is less than $2\theta^i$. Thus we obtain:
    \begin{align}
        \mathbf{x}_r\mathbf{x}_s^{\top}>||\mathbf{x}_r|| \ ||\mathbf{x}_s|| \cos{2\theta^i}=d\cos{2\theta^i},
    \end{align}
    where $\mathbf{x}_r,\mathbf{x}_s$ are arbitrary inputs from class $i$. 
    
    Furthermore, if we set the threshold large enough to ensure that $\theta^i<\frac{r}{4}$, where $r=\min_{i,j} \theta(\mathbf{W}_k^i,\mathbf{W}_k^j)$, we can ensure that for inputs from different classes, $\mathbf{x}_r\mathbf{x}_s^{\top}\leq d\cos{2\theta^i}$.
\end{proof}\hfill $\square$

\textbf{Current linear attention suffers from noise and forgetting due to incomplete use of classification assignments.} As shown in Proposition~\ref{prop:class}, vanilla linear attention implicitly incorporating the notion of information classification, but does not fully leverage the resulting class assignments for compressed storage:
\begin{align}
    \mathbf{S}_t^i &=\sum_{j=1}^c\sum_{s_j\in C_j} w_{s_j}^i \mathbf{x}_{s_j}\mathbf{W}_v.\label{eq:wss}
\end{align}

This mixing of information from different classes leads to two major issues: (1) \textbf{Noise}: Propositions~\ref{prop:noise1} and~\ref{prop:noise2} show that state rows in vanilla linear attention exhibit high similarity due to inter-class interference, resulting in state row homogenization and reduced querying effectiveness. (2) \textbf{Forgetting}: Proposition~\ref{prop:forget} demonstrates that applying uniform decay to all state rows adversely limits the receptive field.

In contrast, by explicitly leveraging class assignments, we introduce a \textbf{row-sparse update formulation of linear attention}:
\begin{align}
    \mathbf{S}_t^i &=\sum_{s_i\in C_i} w_{s_i}^i \mathbf{x}_{s_i}\mathbf{W}_v.\label{eq:wss_sparse}
\end{align}
Through this hard classification, each row of the contextual state stores similar information belonging to the same category, resulting in higher inter-row discriminability and more precise information organization. Simultaneously, each element undergoes fewer decay operations, effectively extending the receptive field over longer contexts. The top-$k$ operation introduced in this work is one of the methods to achieve row-sparse updates.

\begin{definition}\label{def:def1}
    The row-wise cosine similarity in vanilla linear attention is defined as $\operatorname{Cos}\left(\mathbf{S}_t^i, \mathbf{S}_t^j\right)=\frac{\langle \mathbf{S}_t^i, \mathbf{S}_t^j\rangle}{||\mathbf{S}_t^i||\, ||\mathbf{S}_t^j||}$. Similarly, the row-wise cosine similarity under the row-sparse update formulation is denoted as $\operatorname{\widetilde{Cos}}\left(\mathbf{S}_t^1, \mathbf{S}_t^2\right)$ for differentiation.
\end{definition}

\begin{proposition}\label{prop:noise1}
    Assume that $\mathbf{W}_v$ is orthogonal and $w_{s_j}^i\geq 0\, (i,j\in{1,\dots,c})$ in Equation~\eqref{eq:wss}, the row-wise cosine similarity in vanilla linear attention is greater than that of the corresponding row-sparse version, i.e., $\operatorname{Cos}\left(\mathbf{S}_t^i, \mathbf{S}_t^j\right)>\operatorname{\widetilde{Cos}}\left(\mathbf{S}_t^i, \mathbf{S}_t^j\right)$.
\end{proposition}
\begin{proof}
    We use binary-classes as an illustration, which can be generalized to multi-classes scenario. Let $i=1,j=2$, $\overline{\mathbf{S}_t^i} =\sum_{s_1\in C_1} w_{s_1}^i \mathbf{x}_{s_1} + \sum_{s_2\in C_{2}} w_{s_2}^{i} \mathbf{x}_{s_2}$. Because $\mathbf{W}_v$ is an orthogonal matrix, we have:
    \begin{align}
        \langle \mathbf{S}_t^1, \mathbf{S}_t^2\rangle  = \overline{\mathbf{S}_t^1}\mathbf{W}_v\mathbf{W}_v^{\top}\overline{\mathbf{S}_t^2}^{\top}=\langle \overline{\mathbf{S}_t^1}, \overline{\mathbf{S}_t^2}\rangle.
    \end{align}
    Since multiplying by an orthogonal matrix does not change the norm of a vector, therefore:
    \begin{align}
        \operatorname{Cos}\left(\mathbf{S}_t^i, \mathbf{S}_t^j\right) &=\operatorname{Cos}\left(\overline{\mathbf{S}_t^i}, \overline{\mathbf{S}_t^j}\right)\label{eq:S30}
    \end{align}
    We have demonstrated that all inputs from class $j$ have an angle with $\mathbf{W}_k^j$ less than $\theta^j$ (Proposition~\ref{prop:class}). For $\mathbf{x}_{s_j}\in C_j$ and $w_{s_j}^i\geq 0$, the positive weighted sum $\sum_{s_j\in C_{j}} w_{s_j}^{i} \mathbf{x}_{s_j}$ still have an angle with $\mathbf{W}_k^j$ less than $\theta^j$, due to the parallelogram law. Thus we can let $\sum_{s_j\in C_{j}} w_{s_j}^{i} \mathbf{x}_{s_j} = \alpha^i_j \mathbf{x}^i_j$, where $\mathbf{x}^i_j\in C_j$ and $\alpha^i_j>0$ implies module changed ($||\mathbf{x}^i_j||^2=d$).
    
    Then, for vanilla linear attention, we have $\overline{\mathbf{S}_t^1} =\alpha^1\mathbf{x}_1^1 + \beta^1\mathbf{x}_2^1$ and $\overline{\mathbf{S}_t^2} =\alpha^2\mathbf{x}_1^2 + \beta^2\mathbf{x}_2^2$. For notational simplicity, we use $\alpha^i,\beta^i$ representing $\alpha_1^i,\alpha_2^i$, respectively. Let $\cos\theta = \min\{\cos2\theta^1, \cos2\theta^2\}$, we have:
    \begin{align}
        \langle\mathbf{x}_1^1,\mathbf{x}_1^2\rangle &> d \cos2\theta^1 \geq d \cos\theta,\\
        \langle\mathbf{x}_2^1,\mathbf{x}_2^2\rangle &> d \cos2\theta^2 \geq d \cos\theta,\label{eq:S31}\\
        \langle\mathbf{x}_1^1,\mathbf{x}_2^2\rangle &\leq d \cos\theta,\label{eq:final_proof}\\
        \langle\mathbf{x}_2^1,\mathbf{x}_1^2\rangle &\leq d \cos\theta,\\
        ||\overline{\mathbf{S}_t^i}||^2 &= (\alpha^i)^2 d +(\beta^i)^2 d + 2\alpha^i\beta^i\langle\mathbf{x}^i_1,\mathbf{x}^i_2\rangle \leq d(\alpha^i+\beta^i)^2.
    \end{align}
    Because the row index is symmetric, we suppose $\langle\mathbf{x}_2^1,\mathbf{x}_1^2\rangle\geq\langle\mathbf{x}_1^1,\mathbf{x}_2^2\rangle$ without loss of generality. Then we can obtain:
    \begin{align}
    \operatorname{Cos}\left(\overline{\mathbf{S}_t^1}, \overline{\mathbf{S}_t^2}\right) &> \frac{(\alpha^1\alpha^2 + \beta^1\beta^2)(\cos\theta+\epsilon)d+(\alpha^1\beta^2+\beta^1\alpha^2)\langle\mathbf{x}_1^1,\mathbf{x}_2^2\rangle}{d(\alpha^1+\beta^1)(\alpha^2+\beta^2)},\label{eq:S35}
    \end{align}
    where $\epsilon$ refers to a sufficiently small constant. For the corresponding row-sparse linear attention, each state row corresponds to a class, so the contextual state stores information from different classes separately. That is, $\overline{\mathbf{S}_t^1} =\alpha\mathbf{x}_1^1$ and $\overline{\mathbf{S}_t^1} = \beta\mathbf{x}_2^2$. The similarity of the row-sparse attention is computed by:
    \begin{align}
        \operatorname{\widetilde{Cos}}\left(\overline{\mathbf{S}_t^1}, \overline{\mathbf{S}_t^2}\right) &= \frac{\langle\mathbf{x}_1^1,\mathbf{x}_2^2\rangle}{d}.\label{eq:S36}
    \end{align}
    The proof is given by:
    \begin{align}
        d\cos\theta+d\epsilon &> \langle\mathbf{x}_1^1,\mathbf{x}_2^2\rangle, \quad (\text{holds according to Equation~\eqref{eq:final_proof}}) \\
        \Rightarrow (\alpha^1\alpha^2 + \beta^1\beta^2)(\cos\theta+\epsilon)d &> (\alpha^1\alpha^2 + \beta^1\beta^2)\langle\mathbf{x}_1^1,\mathbf{x}_2^2\rangle, \quad (\alpha^i>0, \beta^i>0)\label{eq:S38}\\
        \Rightarrow \operatorname{Cos}\left(\overline{\mathbf{S}_t^1}, \overline{\mathbf{S}_t^2}\right) &> \operatorname{\widetilde{Cos}}\left(\overline{\mathbf{S}_t^1}, \overline{\mathbf{S}_t^2}\right).\quad (\text{holds according to Equations~\eqref{eq:S35}, \eqref{eq:S36}, and \eqref{eq:S38})}\label{eq:S40}
    \end{align}
   This means the state cosine similarity of vanilla linear attention is larger than the corresponding row-sparse version.

    For the multi-classes scenario, we can still utilize triangle inequality of norm $||\overline{\mathbf{S}_t^i}||$ and Proposition~\ref{prop:class} to reach the same conclusion. Note that in this case, we can solely focus on the two most different classes, namely $\langle\mathbf{x}_i^i,\mathbf{x}_j^j\rangle = \min_{k,l}\langle\mathbf{x}^i_k,\mathbf{x}^j_l\rangle$, to complete the derivation, which is enough to prove Proposition~\ref{prop:noise2}.

    Combining Equations~\eqref{eq:S30} and \eqref{eq:S40}, we derive:
    \begin{align}
        \operatorname{Cos}\left(\mathbf{S}_t^i, \mathbf{S}_t^j\right)&>\operatorname{\widetilde{Cos}}\left(\mathbf{S}_t^i, \mathbf{S}_t^j\right)
    \end{align}
\end{proof}\hfill $\square$

Finally, in Proposition~\ref{prop:noise2} we demonstrate that a more precise contextual state (established in Proposition~\ref{prop:noise1}) enables queries to extract diverse information more effectively. For a well-designed classification function $f(\mathbf{x}_t, \mathbf{W}_k)$, we expect the reading operation $\mathbf{q}_t\mathbf{S}_t$ to allow different queries to extract distinct information, a property we refer to as state distinguishability. Higher distinguishability, achieved via row-sparse updates, reflects a more expressive and structured state. In contrast, lower distinguishability, typically observed in vanilla linear attention, leads to state homogenization and inter-class interference, which in turn hampers the model's ability to generate accurate output representations.

\begin{definition}
The measure $\min_{\mathbf{p}_t,\mathbf{q}_t}\operatorname{Cos}(\mathbf{q}_t\mathbf{S}_t,\mathbf{p}_t\mathbf{S}_t)$, referred to as state distinguishability, quantifies the minimum cosine similarity between states outputs, where $\mathbf{q}_t, \mathbf{p}_t$ are arbitrary queries.
\end{definition}

\begin{proposition}\label{prop:noise2}
    For arbitrary queries $\mathbf{q}_t, \mathbf{p}_t\geq\mathbf{0}$, assume that the row norms of $\mathbf{S}_t$ are strictly bounded. The lower bound of the state distinguishability measure, $\min_{\mathbf{p}_t,\mathbf{q}_t}\operatorname{Cos}(\mathbf{q}_t\mathbf{S}_t,\mathbf{p}_t\mathbf{S}_t)$, is given by $\min_{i,j}\langle \mathbf{S}_t^i, \mathbf{S}_t^j\rangle$. Moreover, row-sparse updates reduces this lower bound compared to vanilla linear attention.
\end{proposition}
\begin{proof}
    Considering the cosine similarity, we have:
    \begin{equation}
    \begin{aligned}
        \operatorname{Cos}(\mathbf{q}_t\mathbf{S}_t,\mathbf{p}_t\mathbf{S}_t) &= \frac{\langle\mathbf{q}_t\mathbf{S}_t,\mathbf{p}_t\mathbf{S}_t\rangle}{||\mathbf{q}_t\mathbf{S}_t||\,||\mathbf{p}_t\mathbf{S}_t||}\\
        &\geq \frac{\langle\mathbf{q}_t\mathbf{S}_t,\mathbf{p}_t\mathbf{S}_t\rangle}{(||\mathbf{q}_t||\,||\mathbf{S}_t||)(||\mathbf{p}_t||\,||\mathbf{S}_t||)}\\
        &= \frac{1}{||\mathbf{S}_t||^2} \langle\overline{\mathbf{q}}_t\mathbf{S}_t,\overline{\mathbf{p}}_t\mathbf{S}_t\rangle,
    \end{aligned}
    \end{equation}
    where $\overline{\mathbf{q}}_t = \frac{\mathbf{q}_t}{||\mathbf{q}_t||}, \overline{\mathbf{p}}_t = \frac{\mathbf{p}_t}{||\mathbf{p}_t||}$. Then the lower bound of the state distinguishability measure is given by:
    \begin{align}
        \min_{\mathbf{p}_t,\mathbf{q}_t}\operatorname{Cos}(\mathbf{q}_t\mathbf{S}_t,\mathbf{p}_t\mathbf{S}_t) &\geq \frac{1}{||\mathbf{S}_t||^2}\min_{\mathbf{p}_t,\mathbf{q}_t} \langle\overline{\mathbf{q}}_t\mathbf{S}_t,\overline{\mathbf{p}}_t\mathbf{S}_t\rangle.\label{eq:lower_bound_1}
    \end{align}
    The right side of Equation~\eqref{eq:lower_bound_1} is:
        \begin{equation}
    \begin{aligned}
        \min_{\mathbf{p}_t,\mathbf{q}_t} \langle\overline{\mathbf{q}}_t\mathbf{S}_t,\overline{\mathbf{p}}_t\mathbf{S}_t\rangle&=\min_{||\overline{q}_{t}||=||\overline{p}_{t}||=1}\langle\sum_{i=1}^{c}\overline{q}_{t}^i\mathbf{S}_t^i, \sum_{j=1}^{c}\overline{p}_t^j\mathbf{S}_t^j\rangle\\
        &=\min_{||\overline{q}_{t}||=||\overline{p}_{t}||=1} \sum_{i,j} \overline{q}_{t}^i\overline{p}_t^j \langle \mathbf{S}_t^i, \mathbf{S}_t^j\rangle\\
        &=\min_{i,j}\langle \mathbf{S}_t^i, \mathbf{S}_t^j\rangle
    \end{aligned}
    \end{equation}
    We assume the row norm of $\mathbf{S}_t$ is bounded, which means:
\begin{align}
    \exists\  \epsilon\quad \mathrm{s.t.}\quad \min_i ||\mathbf{S}_t^i||\geq\epsilon.\label{eq:S45}
\end{align}
The assumption implies that the contextual state is numerically bounded, thus the inner-product between state rows is governed by the cosine similarity:
    \begin{equation}
    \begin{aligned}
    \min_{\mathbf{p}_t,\mathbf{q}_t}\operatorname{Cos}(\mathbf{q}_t\mathbf{S}_t,\mathbf{p}_t\mathbf{S}_t) &\geq \frac{1}{||\mathbf{S}_t||^2}\min_{i,j}\langle \mathbf{S}_t^i, \mathbf{S}_t^j\rangle\\
    &= \frac{1}{||\mathbf{S}_t||^2}\min_{i,j} \operatorname{Cos}\left(\overline{\mathbf{S}_t^i},\overline{\mathbf{S}_t^j}\right)\cdot||\mathbf{S}_t^i||\cdot||\mathbf{S}_t^j||,\quad \text{(holds according to Defination~\ref{def:def1})}\\
    &\geq \frac{\epsilon^2}{||\mathbf{S}_t||^2}\min_{i,j} \operatorname{Cos}\left(\overline{\mathbf{S}_t^i},\overline{\mathbf{S}_t^j}\right),\text{(holds according to Equation~\eqref{eq:S45})}
    \end{aligned}
    \end{equation}
which shows that the lower bound of the state distinguishability measure is related to the cosine similarity between the two most different classes. Because row-sparse update formulation has reduced the right-side similarity to $\operatorname{\widetilde{Cos}}\left(\overline{\mathbf{S}_t^i}, \overline{\mathbf{S}_t^j}\right)$, it can be seen as a way to reduce the lower bound of the measure.
\end{proof}\hfill $\square$

In Proposition~\ref{prop:forget}, we demonstrate that the exponential cumulative decay inherent to vanilla gated linear attention adversely affects receptive fields. In contrast, row-sparse update formulation mitigates this issue and can theoretically achieve infinitely long receptive fields, as sparse decay ensures the retention of important information.

\begin{definition}
The attention score of \( \mathbf{x}_t \) with respect to \( \mathbf{x}_s \) is defined as $p_{ts}=\mathbf{q}_t((\prod_{j=s+1}^t\bm{\alpha}_j)\odot\mathbf{k}_s)^{\top}$.
\end{definition}

\begin{definition}
Let $p_{ts}$ denote the attention scores and \( P_{th} \) be a given threshold, where $t\geq s>0$. The receptive field at time $t$ is defined as: $M_t = \max\{t-s\ |\ p_{ts}\geq P_{th}\}$.
\end{definition}

\begin{proposition}\label{prop:forget}
    Let \( \mathbf{q}_t \) and \( \mathbf{k}_s \) be all-ones vectors, and let $t\geq s>0$. For vanilla gated linear attention with decay factors $\bm{\alpha}_t\in(0,1)$, the receptive field is upper-bounded by $M_t= \max_{k=1,\dots, d}\log(\frac{P_{th}}{d}) / \log(\max_{s+1\leq j \leq t} \bm{\alpha}_{j}^k)$, such that the attention score $p_{ts}$ is less than a given threshold \( P_{th} \). Conversely, row-sparse update formulation can retain important information across arbitrary spans, ensuring the attention score remains above the threshold \( P_{th} \), i.e., $M_t=t$.
\end{proposition}
\begin{proof}
    Assume that \( \mathbf{q}_t = \mathbf{k}_s = [1,\dots, 1]\in\mathcal{R}^d\), then the attention score:
    \begin{align}
        p_{ts}&=\mathbf{q}_t((\prod_{j=s+1}^t\bm{\alpha}_j)\odot\mathbf{k}_s)^{\top}=\sum_{k=1}^d\prod_{j=s+1}^t\bm{\alpha}_{j}^k,
    \end{align}
    which can characterize the utility of only decay.

    Effective receptive field requires the attention score $p_{ts}\geq P_{th}$. The corresponding necessary condition is:
    \begin{align}
        \exists\ k,\quad &\text{s.t.}\quad  (\max_{s+1\leq j \leq t} \bm{\alpha}_{j}^k)^{t-s} \geq \frac{P_{th}}{d}.
    \end{align}
    That is, 
    \begin{align}
        \exists\ k,\quad &\text{s.t.}\quad t-s \leq \log(\frac{P_{th}}{d}) / \log(\max_{s+1\leq j \leq t} \bm{\alpha}_{j}^k)
    \end{align}
    
    Therefore, $M_t = \max_{k=1,\dots, d}\log(\frac{P_{th}}{d}) / \log(\max_{s+1\leq j \leq t} \bm{\alpha}_{j}^k)$ represents the upper bound of receptive fields. For earlier tokens at time $s'$, if the receptive field exceeds this upper bound ($t-s'>M_t$), the information is forgotten as $p_{ts}< P_{th}$. 

    In contrast, for row-sparse update variant of linear attention whose decay values can be scattered to actually equal $1$, we can ensure the retention of important information across arbitrary spans by simply considering:
    \begin{align}
        \exists\ k,\quad &\text{s.t.}\quad  \bm{\alpha}_{j}^k=1 \ \text{ for } \ \forall s+1\leq j \leq t.
    \end{align}
    This implies that row-sparse update linear attention can ensure $M_t = t$.
\end{proof}\hfill $\square$

\begin{figure}
    \centering
    \includegraphics[width=0.8\textwidth]{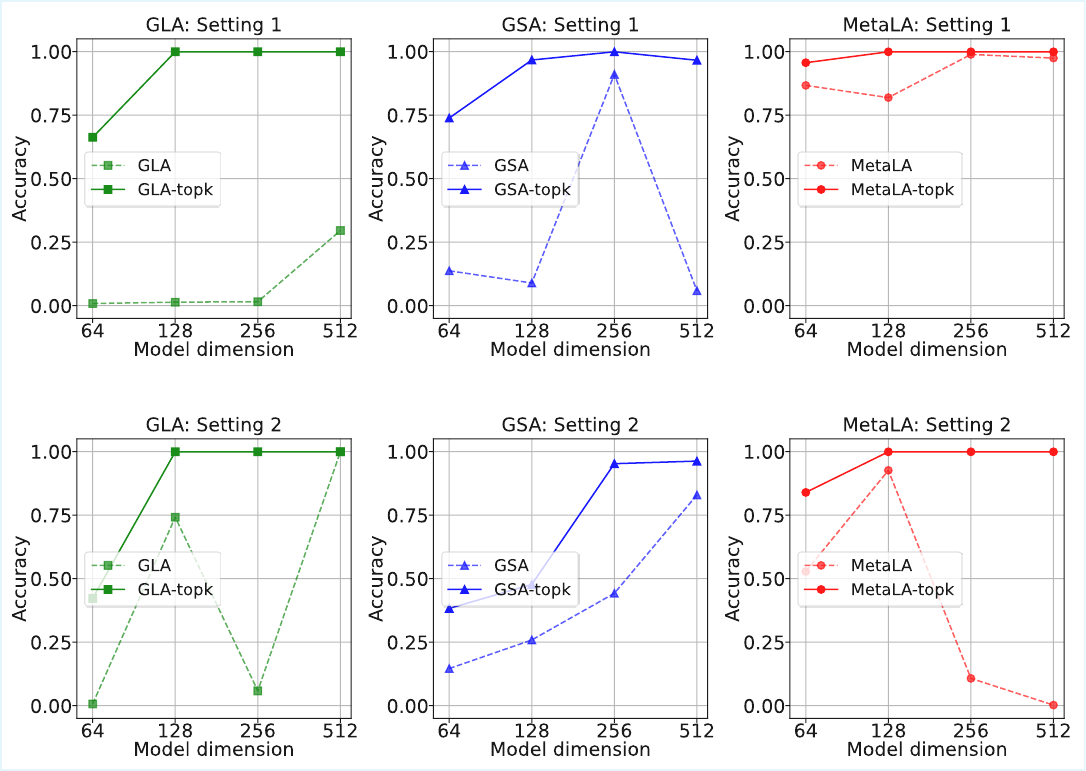}
    \caption{\textbf{MQAR results of linear attention models and their top-$k$ row-sparse variants.} Settings 1 and 2 introduce zero tokens and random noise tokens, respectively. Top-$k$ row-sparse updates lead to longer effective context and more accurate information storage within the state. Moreover, with row-sparse updates, recall performance increases monotonically with respect to state size (i.e., model dimension).}
    \label{fig:mqa_topk}
\end{figure}

\textbf{Synthetic experiment validation.}  
We validate our theoretical analysis by applying top-$k$ sparsity during the state update process of three linear attention models, and evaluate recall performance on the synthetic MQAR task~\citep{arora2024zoology}. Specifically, only the state rows corresponding to the $k$ largest values in $\mathbf{k}_t$ are updated at each step. In MQAR (Multi-Query Associative Recall), the model must retrieve previously stored key-value pairs in response to multiple queries. We evaluate recall under two settings:  
(1) In the first setting, zero tokens are inserted between key-value pairs and queries to explicitly evaluate the model’s effective receptive field.  
(2) In the second setting, random noise tokens are added, testing the model’s robustness to interference during memory retention.

We conduct experiments on GLA~\citep{yang2024gated}, MetaLA~\citep{chou2024metala}, and GSA~\citep{zhang2024gated}—three representative linear attention variants. These models respectively incorporate gated updates, decay-key coupling, and two-pass recurrence. We set $k = d/4$, and train 2-layer models with alternating token and channel mixers, using 2 attention heads. As shown in Figure~\ref{fig:mqa_topk}, the results demonstrate that row-sparse updates offer a simple and general approach for significantly enhancing the recall capabilities of vanilla linear attention, leading to longer effective context and more accurate information storage within the state. Moreover, with row-sparse updates, recall performance increases monotonically with respect to state size (i.e., model dimension).

\section{Pseudocode for Efficient SSE Implementations}\label{app:ops_pseudocode}
SSE expands the contextual state into multiple partitions, each consisting of different subsets of tokens. During operator execution, our goal is to preserve parallelism across partitions, rather than relying on sequential invocation. To minimize unnecessary computational overhead, we introduce two implementations tailored to different sequence length regimes: a masking-based version for short sequences, and a varlen-based version for long sequences.

\textbf{Naive implementation via masking for short contexts.} 
In the chunk-wise operator of linear attention, each kernel instance is responsible for processing a single data chunk. To improve execution efficiency, the chunk size is typically set to a power of two (e.g., 32, 64, or 128), which facilitates memory alignment and GPU-level parallelism. The optimal chunk size depends on both the target GPU architecture and the specific characteristics of the operator, and is typically determined through empirical tuning. However, variable-length training often includes a large proportion of short sequences, which can lead to per-partition lengths falling below the optimal chunk size, or even below the minimum threshold (e.g., 16). In such cases, the varlen implementation incurs additional operational overhead, and grouping tokens by their assigned partitions and processing them with cu\_seqlens can become inefficient. To address this, we adopt a masking-based strategy that increases parallelism through replication while enabling the use of larger chunk sizes. Specifically, each activation is replicated across partitions, and masking is applied based on the top-$k$ selection to ensure that each partition only attends to its assigned tokens. This approach improves GPU utilization without requiring token reordering. The corresponding pseudocode is presented in Algorithm~\ref{alg:masking_impl}.

\begin{algorithm}[htbp]
\caption{SSE Implementation via Top-$k$ Masking}
\label{alg:masking_impl}
\KwIn{$\mathbf{Q}, \mathbf{K}, \mathbf{V} \in \mathbb{R}^{L\times H \times D}$, $\mathbf{E} \in \mathbb{R}^{L \times N}$, cu\_seqlens $\in \mathbb{R}^{S}$, Number of partitions $N$, Top-k partitions $K$}
\KwOut{$\mathbf{O} \in \mathbb{R}^{L\times H \times D}$}
\tcp{Compute top-$k$ selection mask}
$\mathbf{M}_{\text{topk}} = \operatorname{TopkAndMask}(\mathbf{E}, K, \text{dim}=1) \in \{0,1\}^{L \times N}$

\tcp{Repeat input across partitions and apply masking}
$ \mathbf{Q}, \mathbf{K}, \mathbf{V} = \operatorname{Repeat}(\mathbf{Q}, \mathbf{K}, \mathbf{V}, N,\text{dim}=1) \in \mathbb{R}^{L \times N\times H \times D} $\;
$\mathbf{Q} = \mathbf{Q} \odot \mathbf{M}_{\text{topk}}$ \;
$\mathbf{K} = \mathbf{K} \odot \mathbf{M}_{\text{topk}}$ \;
$\mathbf{V} = \mathbf{V} \odot \mathbf{M}_{\text{topk}}$ \;

\tcp{Rearrange for linear attention}
$ \mathbf{Q}, \mathbf{K}, \mathbf{V} = \operatorname{Rearrange}(\mathbf{Q}, \mathbf{K}, \mathbf{V}, \text{dim}=(1,2)) \in \mathbb{R}^{L \times (N H) \times D} $\;

\tcp{Linear attention computation}
$\mathbf{O}=\operatorname{LinearAttention}(\mathbf{Q}, \mathbf{K}, \mathbf{V}, \text{cu\_seqlens})$\;
\Return $\mathbf{O}$\;
\end{algorithm}

\textbf{Efficient implementation via varlen technique for long contexts.} During the long-context continual training phase, the sequence lengths are generally longer and more evenly distributed, allowing chunk-wise computation to operate only on the relevant tokens within each partition. This avoids the redundant computations over masked tokens inherent in the naive masking-based implementation. Specifically, we first derive the top-$k$ partition indices and use them to reorder the QKV vectors, grouping tokens sequentially by partition (from 1 to $N$) within each sample. Next, a new cu\_seqlens is computed based on the reordered sequences and their corresponding partition assignments. At this stage, each resulting subsequence corresponds to a specific partition within a specific sample. Given this reordering and the updated cu\_seqlens, all partitions can be processed in parallel using chunk-wise linear attention, without introducing additional computational overhead. This implementation exhibits favorable scalability with respect to state size $N$, maintaining nearly constant overhead as long as $K$ (the number of selected partitions) remains fixed. The corresponding pseudocode is presented in Algorithm~\ref{alg:varlen_impl}.

\begin{algorithm}[htbp]
\caption{SSE Implementation via Varlen Technique}
\label{alg:varlen_impl}
\KwIn{$\mathbf{Q}, \mathbf{K}, \mathbf{V} \in \mathbb{R}^{L\times H \times D}$, $\mathbf{E} \in \mathbb{R}^{L \times N}$, cu\_seqlens $\in \mathbb{R}^{S}$, Number of partitions $N$, Top-k partitions $K$}
\KwOut{$\mathbf{O} \in \mathbb{R}^{L\times H \times D}$}

\tcp{Obtain reorder index $\mathbf{I} \in \mathbb{R}^{KL}$ and updated sequence offsets}
$\mathbf{I}, \text{new\_cu\_seqlens}= \operatorname{GetIndexAndOffsets}(\mathbf{E}, K, \text{cu\_seqlens})$

\tcp{Group tokens by partition order}
$ \mathbf{Q}, \mathbf{K}, \mathbf{V} = \operatorname{Reorder}(\mathbf{Q}, \mathbf{K}, \mathbf{V}, \mathbf{I},\text{dim}=0) \in \mathbb{R}^{(KL) \times H \times D} $\;

\tcp{Linear attention computation with varlen partitioning}
$\mathbf{O}=\operatorname{LinearAttention}(\mathbf{Q}, \mathbf{K}, \mathbf{V}, \text{new\_cu\_seqlens})$\;
\Return $\mathbf{O}$\;
\end{algorithm}

\textbf{Sequential kernel invocation via for-loop.} For comparative analysis, we present the vanilla sequential implementation. Initially, an additional dimension representing partitions is introduced, and tokens are gathered according to their assigned partitions. Subsequently, the chunk linear attention kernel is sequentially invoked over all partitions within a for-loop, bypassing varlen control. While this approach circumvents the overhead potentially introduced by varlen, it sacrifices inherent parallelism across partitions. Consequently, this method exhibits poor scalability: runtime increases significantly with larger values of $N$ due to its sequential computation.

\section{Experiment Details}\label{app:exp_details}
\textbf{Model Configurations.} Model configurations for our experiments are summarized in Table~\ref{tab:model_hyper}. For SSE, we utilize a single shared partition and set the low-rank dimension to 64 for QK projections. The coefficient for the auxiliary loss is 0.01 across both SSE and MoM. All models employ a Multi-Head Attention (MHA) and SwiGLU~\citep{shazeer2020glu} architecture.

\begin{table*}[h]
\centering
\small
\setlength{\tabcolsep}{4pt}
\centering
\caption{\textbf{Model Architectures.}}
\label{tab:model_hyper}{
\begin{tabular}{cccccccc}
\toprule
Model Size & Non-Embedding Params & Layers & Hidden Dimension & Heads\\
\midrule
600M & 300M & 24 & 1024 & 8  \\
2B  & 1.3B  & 18 & 2304 & 18  \\
\bottomrule
\end{tabular}}
\end{table*}

\textbf{12B SSE-H Conversion.} To further evaluate the scalability of SSE, we convert a pretrained 12B Transformer into an SSE-H variant. The conversion involves a layer-wise stacking pattern of SSE–SWA–SSE–SWA–Softmax layers, applied during the 128k long-context training stage, followed by supervised distillation. As shown in Table~\ref{tab:reasoning_12b}, SSE-H matches the mathematical reasoning accuracy of its softmax-attention counterpart, demonstrating the model’s scalability to larger model sizes.

\begin{table*}[h]
\centering
\small
\setlength{\tabcolsep}{4pt}
\centering
\caption{\textbf{Reasoning ability of the 12B SSE-H variant under the conversion paradigm.}}
\label{tab:reasoning_12b}{
\begin{tabular}{l|ccccccc}
\toprule
\textbf{Model} & AIME24 & AIME25 & MATH500 & OlympiadBench & AMC23 \\
\midrule
Transformer-12B & 75.7 & 58.7 & 95.8 & 82.7 & 97.2 \\
\textbf{SSE-H-n4k1-12B} & 74.3 & 57.3 & 96.0 & 85.3 & 96.3 \\
\bottomrule
\end{tabular}}
\end{table*}

\textbf{Receptive Field Analysis.} To further validate the effectiveness of SSE in long-context modeling, we analyze the receptive field of a pretrained 2B SSE model and compare it with that of a GLA baseline. As visualized in Figure~\ref{fig:crf}, we compute the receptive field across all layers by examining the input-dependent gating matrices after pretraining (with a maximum sequence length of 8k). Specifically, we extract the last 128 tokens and measure the effective receptive field width across different channels in each layer using a threshold of 0.001. The results show that SSE consistently exhibits larger receptive fields than GLA across all layers, confirming SSE's enhanced capacity for long-range information integration.

\begin{figure}
    \centering
    \includegraphics[width=0.8\textwidth]{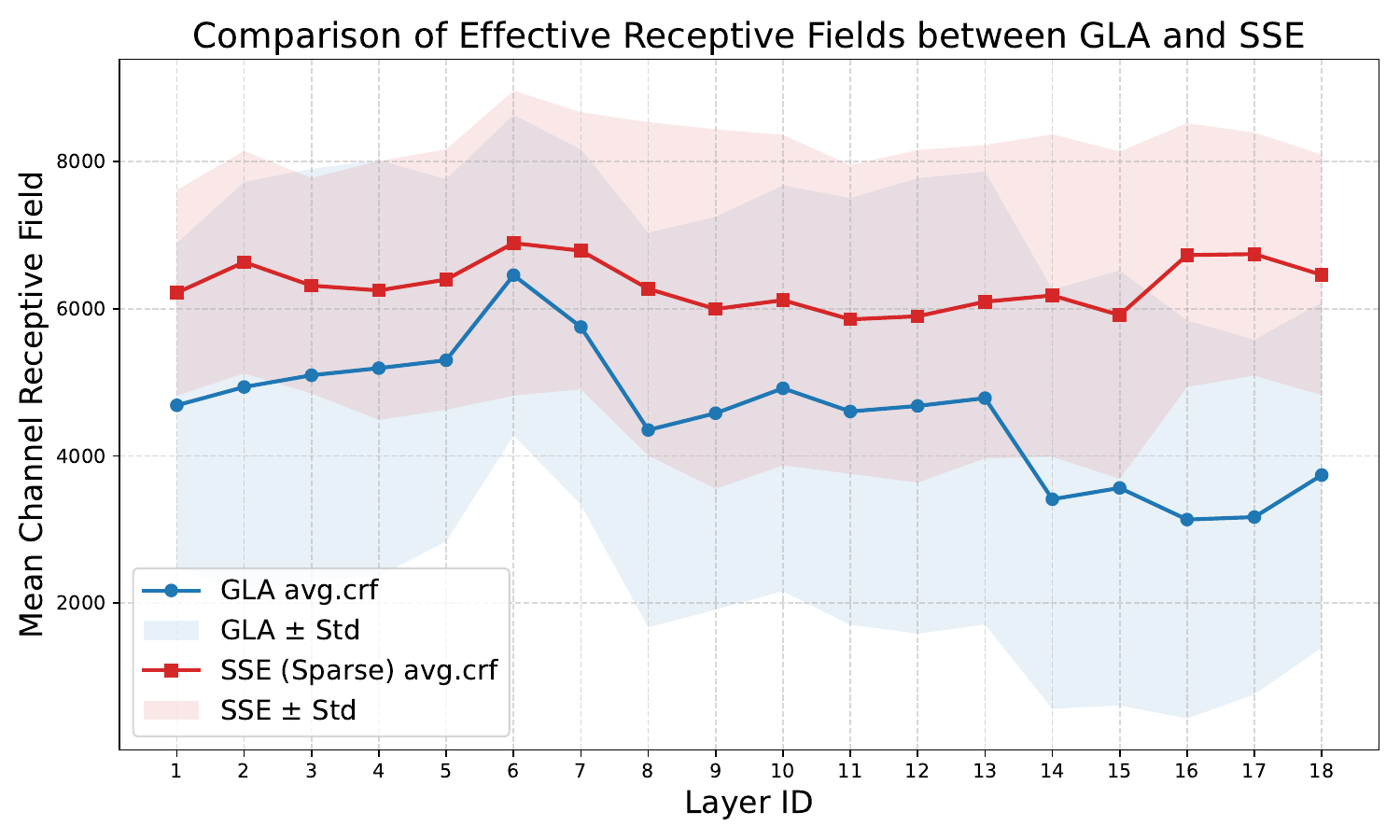}
    \caption{\textbf{SSE exhibits larger receptive fields than GLA.} We visualize the receptive field width of pretrained 2B SSE and GLA models after 8k-seqlen training. Values are computed over the last 128 tokens using a gating threshold of 0.001. SSE shows consistently broader receptive fields across all layers.}
    \label{fig:crf}
\end{figure}

\begin{table*}[t]
\centering
\small
\caption{\textbf{Performance comparison on NIAH tasks in RULER after long-context extension.} All models have 2B parameters with a context length of 32k. Results are reported in a zero-shot setting.}
\label{tab:ruler_32k}

\begin{tabular}{l|ccc|ccc|ccc}
\toprule
\multirow{2}{*}{\textbf{Model}} 
& \multicolumn{3}{c|}{\textbf{S-NIAH-1}} 
& \multicolumn{3}{c|}{\textbf{S-NIAH-2}} 
& \multicolumn{3}{c}{\textbf{S-NIAH-3}} \\
\cmidrule(lr){2-4} \cmidrule(lr){5-7} \cmidrule(lr){8-10}
& 8K & 16K & 32K 
& 8K & 16K & 32K 
& 8K & 16K & 32K \\
\midrule
Transformer 
& 100.0 & 100.0 &  99.4 
& 100.0 & 100.0 & 70.8 
& 97.8 &98.8 &67.4 \\
\textbf{SSE-H-n4k1} 
& 100.0 & 100.0 &  93.6 
& 100.0 & 100.0 &  84.4 
& 100.0 &  98.2 & 90.2 \\
\bottomrule
\end{tabular}

\vspace{0.3cm}

\begin{tabular}{l|ccc|ccc|ccc}
\toprule
\multirow{2}{*}{\textbf{Model}} 
& \multicolumn{3}{c|}{\textbf{MK-NIAH-1}} 
& \multicolumn{3}{c|}{\textbf{MQ-NIAH}} 
& \multicolumn{3}{c}{\textbf{MV-NIAH}} \\
\cmidrule(lr){2-4} \cmidrule(lr){5-7} \cmidrule(lr){8-10}
& 8K & 16K & 32K 
& 8K & 16K & 32K 
& 8K & 16K & 32K \\
\midrule
Transformer 
& 95.8 &90.0  &55.4
& 86.1 & 65.5 &28.6 
& 92.4 & 53.0  &18.1 \\
\textbf{SSE-H-n4k1} 
& 91.0&  84.8 &64.6 
& 89.1 & 67.8 &44.6 
& 87.3 & 66.8  &40.6 \\
\bottomrule
\end{tabular}

\end{table*}

\end{document}